\documentclass[10pt,journal,compsoc]{IEEEtran}

\ifCLASSOPTIONcompsoc
\usepackage[nocompress]{cite}
\else
  \usepackage{cite}
\fi

\ifCLASSINFOpdf
\else
\fi

\usepackage{verbatim}
\usepackage{microtype}
\usepackage{enumitem}
\usepackage{booktabs}
\usepackage{color,graphicx}
\usepackage{comment}
\usepackage{soul}
\usepackage{amsmath,dsfont,bm,mathtools,amsthm,amsfonts}
\usepackage{thmtools, thm-restate}

\usepackage{algorithmic}

\usepackage[linesnumbered,ruled,vlined]{algorithm2e} 
\usepackage{amssymb} 
\usepackage{graphicx} 
\usepackage{bm}
\usepackage[hidelinks]{hyperref}
\SetKwInput{KwInput}{Input}            
\SetKwInput{KwOutput}{Output}

 \newtheorem{theorem}{Theorem}

 \newtheorem{definition}{Definition}
 \newtheorem{lemma}{Lemma}

\newtheorem{axiom}{Axiom}

\usepackage{multirow}
\usepackage{booktabs}
\usepackage{subcaption}
\usepackage{adjustbox}
\usepackage{soul}
\newcommand{\ouralgo}{IW-KSVFair}

\usepackage{subcaption} 
\usepackage{graphicx}   

\usepackage{graphicx}  

\usepackage{tikz}
\definecolor{babypink}{rgb}{0.96, 0.76, 0.76}
\definecolor{bubblegum}{rgb}{0.99, 0.76, 0.8}

\usepackage{cancel}

\newcommand{\nonl}{\renewcommand{\nl}{\let\nl\oldnl}}
\usetikzlibrary{tikzmark,calc}

\usepackage{multirow}

\usepackage{pgfplots}
\pgfplotsset{width=10cm, compat=1.9}
\usepackage{bm}
\usetikzlibrary{calc}

\begin{document}
\title{\LARGE Meritocratic Fairness via $K$-Shapley Values in Budgeted Combinatorial Bandits with Full-Bandit Feedback}

\author{Shradha Sharma, Shweta Jain, and Swapnil Dhamal
\\[1em]
\small{Department
of Computer Science and Engineering\\[0.25em] Indian Institute of Technology Ropar}
}

\def\leftmark{S. SHARMA, S. JAIN, and S. DHAMAL: Meritocratic Fairness via $K$-Shapley Values in Budgeted
Combinatorial Bandits with Full-Bandit Feedback}

\def\rightmark{S. Sharma, S. Jain, and S. Dhamal: Meritocratic Fairness via $K$-Shapley Values in Budgeted
Combinatorial Bandits with Full-Bandit Feedback}

\IEEEtitleabstractindextext{%
\begin{abstract}
We study meritocratic fairness in budgeted combinatorial multi-armed bandits with full-bandit feedback, where a learner selects at most $K$ arms per time step and observes only the noisy aggregate reward of the selected set. To define merit under budgeted coalition constraints, we introduce the $K$-Shapley value, an adaptation of the classical Shapley value that measures marginal contributions using only coalitions of size at most $K$. We show that the $K$-Shapley value is the unique solution concept satisfying symmetry, linearity, null player, and $K$-efficiency axioms. We then establish an $\Omega(T^{2/3})$ lower bound on fairness regret for monotone submodular valuation functions. We show that an explore-then-commit algorithm MURaS (Meritocratic Uniform Random Sampling) achieves $\tilde O(T^{2/3})$ fairness regret by exploring all arms uniformly in exploration phase. To improve empirical regret, we propose \ouralgo, a meritocratic full-bandit algorithm that learns a selection policy whose arm marginals are proportional to the unknown $K$-Shapley values. To correct the bias induced by adaptive sampling, \ouralgo\ uses importance-weighted estimation and mixes the adaptive set distribution with a uniform distribution to keep importance weights bounded. We prove that \ouralgo\ achieves $\tilde O(T^{2/3})$ fairness regret, matching the lower bound up to logarithmic factors. Experiments on synthetic and real-world datasets show that \ouralgo\ achieves low cumulative fairness regret and closely aligns empirical selection frequencies with $K$-Shapley value-based merit.
\end{abstract}

}

\maketitle

\IEEEdisplaynontitleabstractindextext

\IEEEpeerreviewmaketitle

\IEEEraisesectionheading{\section{Introduction}\label{sec:introduction}}

\IEEEPARstart{M}{any} 
sequential decision-making systems repeatedly select a subset of players under a resource budget while observing only the aggregate outcome of the selected set. For example, in social networks, only total influence spread is observed for a selected seed users of a limited size often defined by the budget. The setting can naturally be modeled as budgeted combinatorial multi-armed bandits with full-bandit feedback (BCMAB-FBF)\cite{etc}, where, at each time step, learner selects a feasible subset of arms subject to budget constraint and observes the combined rewards. 

Most bandit algorithms focus on reward maximization which often ends up repeatedly selecting the best performing arms. While this strategy maximizes immediate gain, it can overlook arms that perform nearly as well, leading to winner-takes-all effect. This can reduce participation and engagement from under-selected arms, which may still be valuable. This shows the importance of incorporating fairness into bandit algorithms to ensure that all arms receive a reasonable chance of selection. For instance, in social networks, it ensures that influence budget is spread more broadly rather than concentrated on a few users. 
In other combinatorial decision-making tasks, it helps distribute resources more evenly, preventing over-reliance on a small subset of options.

Under full-bandit feedback, Pokhriyal et al.~\cite{subham2} promote exposure through fixed minimum-pull constraints, which require pre-specified quotas for individual arms. Such quotas may be difficult to determine in practice and can become uninformative when the number of arms is large, as in social network applications. A more adaptive alternative is meritocratic fairness, introduced by Wang et al.~\cite{wangfairness}, under which the selection probability of each arm is proportional to its merit. However, meritocratic fairness has primarily been studied in under semi-bandit feedback (CMAB-SBF), where the individual outcomes of selected arms are observed. Extending this notion to full-bandit feedback is challenging because the learner observes only a noisy aggregate reward for the selected set, making it unclear how an individual arm's merit should be defined and learned.

We define merit via \textit{Shapley value}, a solution concept from cooperative game theory which is known to distribute the reward obtained by a subset of players fairly among the individual players~\cite{shapley1953value}. Computation of Shapley value, however, requires access to the coalition of arbitrary size which is not possible in  BCMAB-FBF because of the budget constraints. 
We introduce
a restricted version of Shapley value, which we call $K$-Shapley value, that considers the marginal contribution of an player by considering coalitions of size at most $K$, with $K$ being the budget. We also show that $K$-Shapley value naturally extends all the key properties of Shapley value in this modified setting.

We next define meritocratic fair policy in BCMAB-FBF as the policy which selects the arms in proportion to the $K$-Shapley value. While the optimal fair policy is defined as the $K$-Shapley-proportional policy, in practical applications, actual $K$-Shapley values can not be computed because of unknown valuation functions and high computational complexity of computing $K$-Shapley value. Therefore, value function must be learnt over time by repeatedly observing the noisy rewards and using Monte Carlo approximation~\cite{datashapley,gog,gtg} to prevent averaging over exponentially many feasible coalitions. 
Naive Monte-Carlo approximation however may lead to unfair policy because arm selection policy must follow $K$-Shapley-proportional policy. On the other hand, adaptive fair policy (based on estimated $K$-Shapley value) may lead to biased estimates of $K$-Shapley value.
To address this, we use importance-weighted estimation, where each observed set is reweighted by its selection probability, and mix the learned adaptive distribution with a uniform distribution to keep the weights bounded. The mixing coefficient is chosen to optimize the resulting regret trade-off. Based on this idea, we propose Importance-weighted $K$-Shapley value fair (\ouralgo) algorithm, which combines full-bandit learning with importance-weighted $K$-Shapley estimation and distribution mixing. We prove that \ouralgo\ achieves fairness regret $\tilde{O}(T^{2/3})$.

We complement the upper bound with a matching lower-bound showing that any algorithm must incur $\Omega(T^{2/3})$ fairness regret in the full-bandit setting, thereby matching our upper bound in its dependence on $T$ up to logarithmic factors. We also show that exploration separated meritocratic uniform random sampling(MURaS) algorithm also achieves the  fairness regret of $\tilde{O}(T^{2/3})$. However, because MURaS uniformly explores all arms during its exploration phase, it incurs substantially larger empirical regret than \ouralgo\ on both the synthetic and real-world datasets. Our contributions are:
\begin{itemize}
    \item We introduce the notion of $K$-Shapley value for BCMAB-FBF and use it to define meritocratic fairness under full-bandit feedback. We also show that it preserves the key properties of the classical Shapley value.

    \item We propose \ouralgo\ for learning $K$-Shapley-proportional policies under full-bandit feedback. We prove that it achieves a fairness regret of $\tilde{O}(T^{2/3})$.

    \item We prove an $\Omega(T^{2/3})$ lower bound on fairness regret, showing that the regret achieved by \ouralgo\ is optimal up to logarithmic factors.

    \item We validate the proposed framework on synthetic dataset and social influence maximization dataset,
    where \ouralgo\ achieves lower fairness regret than the baselines.
\end{itemize}

\section{Related Work}
\noindent\textbf{CMAB-FBF.}~
The CMAB-FBF has been widely studied in the context of reward maximization under the setting of linear reward functions~\cite{Rejwan} and submodular rewards~\cite{etc,online-learning-offline-greedy,randomized}. Rejwan and Mansour~\cite{Rejwan} studied subset selection under a linear reward setting and introduced the Combinatorial Successive Accepts and Rejects algorithm, which achieves a regret bound of $\tilde{O}(K M^{1/2} T^{1/2})$. DART~\cite{dart} proposed an elimination-based algorithm designed for combinatorial bandits with non-linear reward functions under restricted settings. Their analysis shows that DART attains a regret bound of $\tilde{O}(K^{3/2} M^{1/2} T^{1/2})$. In another line of work, Agarwal et al.~\cite{agarwal2021stochastic} considered the setting where reward function is non-linear assuming that the reward function is an elementwise, strictly increasing function of the constituent arm rewards and that the first-order stochastic dominance assumption holds, which guarantees a regret bound of $\tilde{O}(K^{1/2} M^{1/3} T^{2/3})$. Furthermore, in the single-objective bandit feedback setting, when no special reward structure (such as linearity) is exploited, the best known regret bounds that avoid combinatorial dependence are of the order $\tilde{O}(T^{2/3})$~\cite{nie2023framework}. Additionally, Tajdini, Jain, and Jamieson~\cite{tajdini} established a lower bound of $\Omega(T^{2/3})$ for submodular maximization under bandit feedback when the performance is evaluated against greedy benchmarks.

\noindent\textbf{Fairness in CMAB-SBF.}~
CMAB with concave rewards and knapsack constraints~\cite{concave_reward} ensures that each arm receives a minimum level of exposure. 
Similarly, Liu et al.~\cite{constraint} introduces the UCB-LP algorithm to handle long-term linear constraints such as budget and fairness, guaranteeing logarithmic regret with no constraint violations. Fairness in different settings such as sleeping bandit~\cite{sleeping_bandit} and delayed feedback setting~\cite{delayedfeedback} are also considered.
These studies provide a strong foundation for fairness in semi-bandit feedback settings and offer $O(\ln T)$ regret  which is not achievable under full-bandit feedback. 
\\[-0.5em]

\noindent\textbf{Shapley Value and Bandits.}~
\label{sec:rel_shapley}
The Shapley value has been widely used in feature attribution, crowdsourcing, resource allocation, social networks, and federated learning~\cite{shapley_ml,shapley_and_casual,neuroips,shapley_q_value,resourceallo1,resourceallo2,resourceallo3,social_shapley1,social_shapley2,shapley_fl_2,gog,shapley_fl_3}. 
Existing works for restricted cooperative games either assign values to infeasible coalitions~\cite{willson}, assume that the grand coalition remains feasible~\cite{albizuri}, or impose graph-based feasibility constraints~\cite{derks}; none directly capture coalition feasibility determined solely by the budget $K$. Our formulation is also closely related to both semivalues \cite{weber1988probabilistic} and values derived from induced subgames\cite{akimov2001average}. However, semivalues generally relax the classical efficiency axiom, whereas our characterization imposes a budget-adapted $K$-efficiency condition that uniquely determines the allocation over feasible coalitions. Similarly, induced-subgame approaches typically assume that the valuation is defined for all subgames and average across coalition sizes. 
Finally, Shapley-bandit works~\cite{shapley_bandit_game,top_k} focus on ranking or identifying top contributors and typically assume exact access to coalition values, whereas we learn Shapley-proportional fair selection policies from noisy full-bandit feedback.

\section{Preliminaries and Problem Statement}

In a Budgeted Combinatorial Multi-armed Bandit with full bandit feedback (BCMAB-FBF), at each time step $t$, the learner selects a subset of arms $S_t$ with $|S_t| \le K$ from a finite set of arms $[M] = \{1,2,\dots,M\}$ and receives a stochastic reward $V_t(S_t)$. 
The set of all feasible subsets is denoted by $\mathcal{A} = \{ S \subseteq [M] : |S| \le K \}$. We further denote $\pi_t(i)$ as the probability of picking an arm $i$ at time $t$ by a policy $\pi_t$ that is output by the algorithm. The expected reward of each $S \in \mathcal{A}$ is given by the set function  $V : \mathcal{A} \rightarrow [0,1]$ with the convention that $V(\emptyset) = 0$.
This valuation function is typically unknown in real-world applications and must be learnt over a period of time. For each $S \in \mathcal{A}$, let $D_S$ denote a fixed but unknown reward distribution supported on $[0,1]$ with mean $V(S)$. Whenever $S$ is selected, the learner observes an independent sample $\hat{V}(S)$ from $D_S$.

\subsection{$K$-Shapley Value}

The pair $([M], V)$ defines a coalitional game, where $V(S)$ represents the total reward achievable by coalition $S$. 
For any player $i \in [M]$, the Shapley value $\phi_i$ is the average marginal contribution of player $i$ over all possible orderings in which players can be added in the grand coalition. Mathematically, 
$$
\phi_i = \sum_{S \subseteq [M] \setminus \{i\}} \frac{|S|! \, ({M} - |S| - 1)!}{M!} \left( V(S \cup \{i\}) - V(S) \right)
$$

It is easy to see that for monotone valuation function, $\phi_i \ge 0, \forall i\in [M]$. 
Let $([M],V,K)$ denote a restricted cooperative game (R-Game) where the valuation function is defined only for coalitions with size bounded by $K$. \textit{$K$-Shapley value}, extends the Shapley value concept to R-Game.
For any base player $i \in [M]$, let its marginal contribution to a coalition $S \subseteq [M] \setminus \{i\}$ of size $\le K-1$ be defined as $\Delta_i(S) = V(S \cup \{i\}) - V(S)$.
Let $\psi_i^K(S_k) = \sum_{S \subseteq S_k \setminus \{i\}}
\frac{|S|! \, (K - |S| - 1)!}{K!}\Delta_i(S)$ represent the average of the marginal contributions of $i$ for a set $S_k$ across all possible coalitions $S\subseteq S_k\setminus \{i\}$. Then, the $K$-Shapley value $\phi_i^K$ denotes the expected value of $\psi_i^K(S_k)$ over all possible such sets $S_k$ of size $K$ and containing $i$, i.e.,
\begin{align*}
\phi_i^K = 
\sum_{S_k \subseteq [M] : |S_k| = K, i \in S_k}
\frac{1}{\binom{M - 1}{K - 1}}\psi_i^K(S_k)
\end{align*}

\noindent $K$-Shapley value naturally extends Shapley value's axiomatic properties: symmetry, linearity, and Null player. These properties are restated in Appendix~\ref{appendix:proof_of_uniqueness}. 
The efficiency condition of Shapley value requires $\sum_{i\in[M]} \phi_i = V([M])$. Since $V([M])$ is not defined in BCMAB-FBF, it requires a natural modification. Consider two quantities to define efficiency:
\begin{itemize}[leftmargin=*]
\item Expected valuation of $K$-sized coalitions by assuming each coalition is equally likely, given as: $
\frac{1}{\binom{M}{K}} 
\sum_{\substack{T_K \subseteq [M]\\ |T_K|=K}} V(T_K)
$. 

\item Expected total $K$-Shapley contribution in a random $K$-coalition given as 
$\mathbb{E}_{T_K \subseteq [M]:|T_K|=K}\left[\sum_{i \in T_K} \phi_i^K\right]
= 
\frac{1}{\binom{M}{K}} \binom{M-1}{K-1} \sum_{i \in [M]} \phi_i^K$.
Here, $\binom{M-1}{K-1}$ arises because 
each player’s $K$-Shapley value contributes to exactly $\binom{M-1}{K-1}$ coalitions of size $K$.
\end{itemize}

\noindent
For the efficiency to hold, the expected total Shapley contributions across all players must be consistent with the average valuation of $K$-coalitions. Therefore, by equating the two quantities and simplifying, we obtain the following.
\begin{axiom}[$K$-Efficiency]
The $K$-Shapley value satisfies the \textit{$K$-efficiency} property if 
$\sum_{i \in [M]} \phi_i^K =
\frac{1}{\binom{M - 1}{K - 1}}
\sum_{\substack{T_k \subseteq [M] \\ |T_k| = K}} 
V(T_k).$
\end{axiom}

\begin{theorem}
\label{thm:kshapley}
For valuation function, $V: \mathcal{A} \rightarrow \mathbb{R}, V(\emptyset) =0$, the $K$-Shapley value is the unique solution concept satisfying symmetry, linearity, null player, and K-efficiency.
\end{theorem}

\noindent
\textit{Proof sketch.}
The proofs of the linearity and the null player properties follow along similar lines as in the classical Shapley value, as the $K$-restriction preserves the fundamental structure of these axioms. 
However, the proofs of symmetry and $K$-efficiency must be handled separately owing to the restriction on coalition sizes. To prove the uniqueness, we extend the notion of carrier games~\cite{narahari,maschler} to that of $K$-restricted carrier games. The detailed proof of both parts is provided in Appendix~\ref{appendix:proof_of_uniqueness}. 

The class of $[0,1]$-valued games is
not closed under $\alpha V_1+\beta V_2$ for arbitrary
$\alpha,\beta\in\mathbb{R}$. We therefore prove uniqueness
on the space of real-valued restricted games and impose boundedness only for the fairness policy.

\subsection{Meritocratic Fairness}
Let $\mathcal{S}_{M,K} = \{p\in[0,1]^M: \sum_{i=1}^M p_i = K\}$ define the simplex of all possible policies for selecting $K$ out of $M$ arms. Then, meritocratic fairness is defined as follows.

\begin{definition}[\textbf{Meritocratic Fairness}]
A policy $\pi \in \mathcal{S}_{M,K}$ is \emph{meritocratic fair} if, for every
$t\in[T]$, there exists $c_t\ge 0$ such that $\pi_t(i)=c_t\phi_i^K,\; \forall i\in[M]$. 
\end{definition}

We assume that $V$ is a monotone function with $V(\emptyset) = 0$, thus leading to non negative Shapley value.
The definition coincides with the meritocratic fairness proposed by \cite{wangfairness} with merit function being $K$-Shapley value when single arm needs to be selected at each time step. 
When the merit function is (a) bounded, i.e., $\forall i,j:\; \frac{\phi_i^K}{\phi_j^K} \le \frac{M-1}{K-1}$ for $K>1$  and is (b) Lipschitz continuous, then the unique optimal fair policy is given by $\pi^*(i)=\frac{K\phi_i^K}{\sum_{j\in[M]}\phi_j^K}$\cite{delayedfeedback} with $\pi^* \in \mathcal{S}_{M,K}$. Note that in practical applications, candidate arms would typically be competitive, such that  
Shapley value of no arm exceeds $1/K$ of the sum of Shapley values or $M/K$ of the average Shapley value with $M >> K$. Therefore, in the rest of the paper, we also assume that the optimal policy satisfies $\pi^* \in \mathcal{S}_{M,K}$ for theoretical proofs. We validate this assumption in our experiments as well.

\begin{definition}[\textbf{Query-Normalized Fairness Regret}]
At time $t$, the learner evaluates a set $S_t$ with $r_t=|S_t|$. Let $\tau(t)$ denote the index of the policy
update whose resulting policy was used to generate $S_t$,
and let $\mathcal H_{\tau(t)-1}$ denote the history available
when that policy was computed. Define $\pi_t^{(r_t)}(a) = P(a\in S_t \mid \mathcal{H}_{\tau(t)-1}, r_t)$ and $\pi^{*(r_t)}(a)
=
\frac{r_t}{K}\pi^*(a)$. Then, the query-normalized fairness regret over $T$ time steps is:
\[
FR_T
=
\mathbb{E}\left[
\sum_{t=1}^T
\frac{K}{r_t}
\sum_{a=1}^M
\left|
\pi_t^{(r_t)}(a)-\pi^{*(r_t)}(a)
\right|
\right].
\]
The expectation is taken over the randomness of the algorithm,
the sampled coalitions, and the reward observations.
\end{definition}
The conventional fairness regret comparing $|\pi_t - \pi^*|$  implicitly assumes that
both the benchmark policy and the learner selects exactly $K$ arms in
every round. Estimating a $K$-Shapley value, however, requires evaluating coalitions of different sizes. Comparing an $r$-arm query directly with $K$-arm query would be inappropriate because
the two vectors have different total marginal mass. We therefore associate the full-budget fair policy $\pi^*$ with the
size-$r$ fair marginal vector $\pi^{*(r)}$.
This preserves the relative merit shares while ensuring that both the
learner and the benchmark allocate exactly $r$ units of marginal
probability. However, this leads to  the unnormalized discrepancy $\sum_{a=1}^M|
\pi_t^{(r)}(a)-\pi^{*(r)}(a)|$ that 
scales linearly with $r$. 
To prevent the fairness loss from being artificially reduced solely
by evaluating a smaller coalition and ensuring that every full-bandit
query receives the same weight, the discrepancy is normalized by $K/r$.

For example, in a social network recommendation system, one recommendation round may expose only a small number of posts, while another may use the full recommendation budget. Suppose both rounds exhibit the same relative bias by overexposing content from highly connected users and underexposing content from less-connected users with comparable relevance. These two rounds have the same compositional unfairness, even though the second contains more recommendation slots. An unnormalized discrepancy would assign a smaller fairness loss to the first round merely because fewer slots are available. Query normalization instead compares the fraction of recommendation slots assigned to each user with the user’s merit-based exposure share, and therefore assigns the same fairness loss to the same relative bias.

\subsection{Lower Bound on Fairness Regret}

We now prove the lower bound for $K<M$ showcasing the inherent 
difficulty in learning fair policies under CMAB-FBF.

\begin{theorem}[Lower Bound]
\label{thm:k2-fairness-lower-bound}
Consider BCMAB-FBF with fixed $M$ and $K$ with $1 < K < M$. Then, there exists a monotone  submodular valuation function satisfying $\epsilon \le \phi_i^K\ \forall i\in [M]$ such that, for every algorithm,
$   \text{sup}_V \mathrm{FR}_T(V)
    =
    \Omega(T^{2/3})$. 
\end{theorem}

\noindent
\textit{Proof sketch.}
The proof constructs two monotone submodular instances that are identical
except for the value of one size-$K$ coalition. This small perturbation
changes the corresponding fair policies by $\Theta(\Delta)$ in the two instances leading to repeatedly querying the informative coalition. 
Each query provides only $O(\Delta^2)$
statistical information. Therefore, either the learner queries the
informative coalition often enough and incurs large regret in the base
instance, or it cannot reliably distinguish the two instances and
incurs $\Theta(\Delta)$ regret per round in the alternative instance.
Choosing $\Delta=\Theta(T^{-1/3})$ balances these two cases and yields
an $\Omega(T^{2/3})$ lower bound. The detailed proof is provided in  Appendix~\ref{appendix:proof_of_lower_bound}.

Lower bound of $\Omega(T^{2/3})$ exists under reward maximization for CMAB-FBF~\cite{tajdini2024nearly}. However, extending the existing construction is not trivial;
we require an additional Shapley-separation argument showing that the hidden perturbation induces a sufficiently large change in fair policy and that querying the informative sets is costly in fairness regret. Additionally, under the assumption that arms can be ordered, DART~\cite{dart} provides a rejection-based algorithm with regret of $O(\sqrt{T})$. A rejection-based policy will incur a linear regret under meritocratic fairness because even a low-ranked arm must be selected with non-zero probability proportional to its merit.

\section{Proposed Approach}
This section introduces exploration-separated algorithm Meritocratic Uniform Random Sampling (MURaS) which first estimate the arms' Shapley values through $R$ times and then select the arms in proportion to these values for the remainnig time. The regret can be derived by Hoeffding's inequality to bound the difference in estimated and exact $K$-Shapley value. The detailed algorithm and regret proof are presented in Appendix~\ref{appendix:muras_algo} and ~\ref{appendix:muras_fairness_regret} respectively.

\begin{theorem}[MURaS Fairness Regret]
Assume that for every arm $a\in[M]$, the $K$-Shapley value satisfies
$\phi_a^K\in[0,1]$, and let $\sum_{a=1}^M \phi_a^K\ge \gamma>0$. 
Then, with $R=\tilde \Theta(T^{2/3})$ 
the cumulative fairness regret satisfies $\mathrm{FR}_T
    =
    \tilde O\left(\frac{KM}{\gamma}T^{2/3}\right)$.
\label{thm:muras}
\end{theorem}

\noindent

MURaS finds Shapley value of every arm uniformly and thus  incurs a huge fairness regret during the exploration rounds. 
We next present \ouralgo, which combines exploration and exploitation through adaptive optimistic updates. While asymptotic fairness regret still remains the same, i.e,  $\tilde O(T^{2/3})$, 
our experiments also show that the \ouralgo\ achieves much lesser fairness regret as compared to MURaS.

\begin{algorithm}[h!]
\caption{\textbf{IW-KSVFair: Importance-Weighted $K$-Shapley Value Fairness}}
\label{algo:iw_ksvfair}
\KwInput{$T,M,K,\delta,\epsilon$}
\KwOutput{Selected sets $S_t$ and policies $\pi_t$}

Initialize $N_{a}\gets 0$, $\hat\phi_a\gets 0\ \forall a\in[M]$, $\mathcal{P}_{t,0} = \phi\ \forall t$\

$\eta \gets \Theta(\frac{M^{2/3}}{K^{1/3}T^{1/3}} )$\

\For{$t=1$ \KwTo $\lfloor T/K \rfloor$}{
    \For{$a\in[M]$} { \label{line:4}
    
        $\hat{\phi}_{t,a}\gets
\min\{\max\{\bar\phi_{t,a},\epsilon\},1\} \label{line:5}$\;
$
        \tilde\pi_t(a)
        \gets
        \frac{K\hat{\phi}_{t,a}}{\sum_{j=1}^M\hat{\phi}_{t,j}} \label{line:6}
$\;
$\pi_t = arg\min_{p\in \mathcal{S}_{M,K}}||p-\tilde\pi_t||_1$ 
\;
    }
Let $P_t$ be a distribution over $K$-sets with marginals $\pi_t$\ \label{line:7};

    $P_t^{\mathrm{mix}}(S)
        =
        (1-\eta)P_t(S)
        +
        \eta\frac{1}{\binom{M}{K}},
        \quad |S|=K$\; \label{line:8}
    
    $q_t(a) = (1-\eta)\pi_t(a) + \eta\frac{K}{M}$\; \label{line:9}
Sample $S_t\sim P_t^{\mathrm{mix}}$\; \label{line:10}

    Sample a permutation $\sigma_t = (a_{t,1}, a_{t,2}, \ldots, a_{t,K})$ uniformly at random from all permutations of $S_t$\; \label{line:11}

    \ForEach{$r\in [K]$}{
    
    Define $\mathcal{P}_{t,r} = \{a_{t,1}, a_{t,2}, \ldots, a_{t,r}\}$ to be the prefixes for arm $a_{t,r}$\;

    Pull $\mathcal{P}_{t,r}$ and observe $\hat{V}(\mathcal{P}_{t,r})$\;
        
        $\Delta_t(a_{t,r})\gets 
        \hat{V}\!\left(\mathcal{P}_{t,r}\right)
        -\hat{V}\!\left(\mathcal{P}_{t,r-1}\right)
        $, $w_t(a_{t,r},S_t)
            \gets
            \frac{q_t(a_{t,r})}{\binom{M-1}{K-1}P_t^{\mathrm{mix}}(S_t)}$\; \label{line:15}
        $Z_t(a_{t,r})\gets w_t(a_{t,r},S_t)\Delta_t(a_{t,r})$, $N_{a_{t,r}}\gets N_{a_{t,r}}+1$\; \label{line:16}
$\bar\phi_{a_{t,r}} \gets \frac{(N_{a_{t,r}}-1)\bar{\phi}_{a_{t,r}} + Z_t(a_{t,r})}{N_{a_{t,r}}}$\;\label{line:17}
    }
}
$t' \gets Kt$\;
\For{$t=t'+1$ \KwTo $T$}{$S_t \sim P_{t'}^{mix}$} 
\end{algorithm}

\subsection{\ouralgo : Importance Weighted $K$-Shapley Fair Algorithm }

Algorithm~\ref{algo:iw_ksvfair} presents \ouralgo, an importance-weighted algorithm for learning meritocratic fair policies. At each round, the algorithm first computes clipped non-negative estimates $\hat{\phi}_{t,a}$ of the $K$-Shapley values and converts them into an arm-selection policy $\pi_t$ (Lines~\ref{line:4}--\ref{line:6}). The distribution $P_t$ is then chosen over feasible $K$-sets so that its marginals are given by $\pi_t$ (Line~\ref{line:7}). To ensure sufficient exploration and bounded importance weights, $P_t$ is mixed with the uniform distribution over all $K$-subsets using the mixing parameter $\eta$, resulting in the mixed distribution $P_t^{\mathrm{mix}}$ (Line~\ref{line:8}). In particular, the uniform mixture guarantees that
$
P_t^{\mathrm{mix}}(S)\geq \frac{\eta}{\binom{M}{K}},
$
which implies that the importance weights are bounded by
$
W=\frac{M}{K\eta}.
$ 
The marginal probability of an arm $a$  under the mixed distribution is $q_t(a)$ (Line~\ref{line:9}). The algorithm then samples a set $S_t$ from $P_t^{\mathrm{mix}}$ (Line~\ref{line:10}). 

The challenge is that the $K$-Shapley value is an expectation over all $K$-subsets containing an arm and over all permutations within those subsets. Therefore, using unweighted marginal contribution  would  produce a biased estimate of the true $K$-Shapley value. 
The algorithm samples a random permutation of the selected set and computes the arm's marginal contribution to its predecessor coalition (Lines~\ref{line:11}--\ref{line:15}) and corrects the bias, the algorithm uses importance weighting (Line~\ref{line:16}). This converts samples drawn from the algorithm's adaptive set distribution into conditionally unbiased samples of the $K$-Shapley value, given that the arm is selected. 
The choice of $\eta$ balances the variance introduced by importance weighting
and the fairness regret introduced by uniform exploration. 
Finally, the algorithm updates the running estimate $\bar{\phi}_a$ of the $K$-Shapley value of arm $a$ (Lines~\ref{line:17}).

\begin{theorem}
Assume that the $K$-Shapley values satisfy $\phi_a^K \in [\epsilon,1]$ for all $a\in[M]$. 
Assume that at each round $t$, the distribution $P_t$ over $K$-sets is exactly computable with marginals $\pi_t$. Then, for
$
    \eta
    =
    \Theta(\frac{M^{2/3}}{K^{1/3}T^{1/3}} )
$, the expected cumulative fairness regret of \ouralgo\ satisfies
\[
\mathrm{FR}_T
=
\widetilde{O}\left(
\frac{C}{\epsilon}K^{2/3} M^{2/3} T^{2/3} + \frac{C}{\epsilon}M^{4/3}K^{1/3}T^{1/3}
\right).
\]
Here, $C = \max_t c_t$ with $c_t = \max\{1,||\tilde{\pi}_t||_\infty\}\le K$.
\end{theorem}

\noindent
\textit{Proof sketch.}
The proof has three main steps. First, proving that conditioned on arm $a$ being selected,
the weighted marginal contribution is an unbiased estimate of its
true $K$-Shapley value. Second, bounding the importance weight by
$W=M/(K\eta)$ using the uniform mixing component guarantees. Third controling the estimation error at the rate $\widetilde{O}(\sqrt{W/N_a}+W/N_a)$ using a Freedman-type concentration inequality. Third, every prefix query generated in a macro-round has
query-normalized fairness loss equal to $\|q_t-\pi^*\|_1$.
Summing over the $K$ prefixes and applying the standard counting
argument gives $\widetilde{O}\!\left(
\frac{CK}{\epsilon}\sqrt{WMT}
+\frac{CKWM}{\epsilon}
+\eta KT
\right)$. Finally, substituting $W=M/(K\eta)$ and balancing the leading
estimation and mixing terms gives the stated bound for
$\eta=\Theta(\frac{M^{2/3}}{K^{1/3}T^{1/3}} )$.
The detailed proof is provided in Appendix~\ref{appendix:regret_iwksvfair}.

The assumption $\phi_a^K \geq \epsilon$ is natural when the candidate set is restricted to nodes with non-negligible merit. For instance, under the Independent Cascade model for influence spread in social networks, adding arm (node) $a$ to a seed set $S$ has marginal contribution $\Delta_a(S)$ that is typically strictly positive because, without being selected as a seed, the probability that $a$ is activated by the other seeds is generally less than one. 
The condition $\phi_a^K \geq \epsilon$ strengthens strict positivity by excluding arms whose influence contribution is negligible, via preliminary screening based on centrality, estimated marginal spread, etc. 
Hence, the assumption applies only to the retained candidate arms and ensures that the merit-proportional policy is well defined and statistically learnable.

\section{Experimental Analysis}
We conduct experiments on a synthetic dataset, and a real-world \textit{Social Influence Maximization} (SIM) dataset.
In SIM, each node acts as a potential influencer, and the goal is to select a seed set to generate diffusion under the Independent Cascade model~\cite{tardos}.
Conventional influence maximization methods tend to repeatedly select a small
group of nodes with high estimated marginal spread, while ignoring other nodes with positive marginal  contributions.
In contrast, \ouralgo\ 
learns selection marginals
proportional to the nodes' $K$-Shapley values,
which leads to more diverse and fair seed selection.

\subsection{Benchmark Algorithms}

\begin{itemize}[leftmargin=*]

    \item \textbf{Uniform Random Sampling (URS):} URS selects a set uniformly from all $\binom{M}{K}$ sets of size $K$. Therefore, every arm has marginal selection probability $K/M$.

    \item \textbf{NoMix:} NoMix is an ablation of IW-KSVFair in which uniform mixing is removed by setting $\eta=0$. It retains the importance-weighted update but directly samples from the adaptive set distribution.

    \item\textbf{NoIW:} 
    NoIW is an ablation of IW-KSVFair that removes only the
    importance-weighted correction. It uses the same mixing
    parameter $\eta$, the same adaptive set distribution $P_t$,
    and the same mixed distribution.
    However, NoIW directly uses the observed marginal
    contribution 
    $
    Z_t(a)=\Delta_t(a)
    $.

    \item \textbf{Fair-CMAB~\cite{subham2}:} 
    Fair-CMABaddresses fairness by guaranteeing the minimum number of pulls for each arm, which is fixed in advance. The fairness thresholds were generated uniformly at random from the interval $\left[\, 0,\; \frac{1}{2\lceil M/K \rceil} - 0.01 \,\right]$~\cite{subham2}.
    Although it enforces quota-based fairness, it remains the only existing work that addresses fairness under the FBF. We include it to contrast static fairness constraints with our meritocratic fairness framework.

    \item \textbf{Explore-then-commit Greedy (ETCG)~\cite{etc}:}
    ETCG is an exploration separated algorithm which first explores the candidate base arms and then greedily commits to the $K$ arms 
    with the highest estimated marginal gain. 

    \item \textbf{Gap-based Exploration (GAP-E)~\cite{top_k}:} GAP-E maintains the  rankings of arms based on estimated Shapley values in each time step and focuses on exploring top-$K$ arms.
\end{itemize}

\begin{figure*}[h!]

    \vspace{0.6em}

\hspace{0.5em}
\includegraphics[
        width=0.99\textwidth
    ]{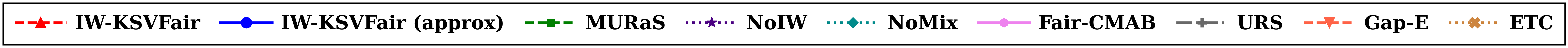}

    \centering
    
    \begin{subfigure}[t]{0.32\textwidth}
        \centering
        \includegraphics[width=\textwidth]{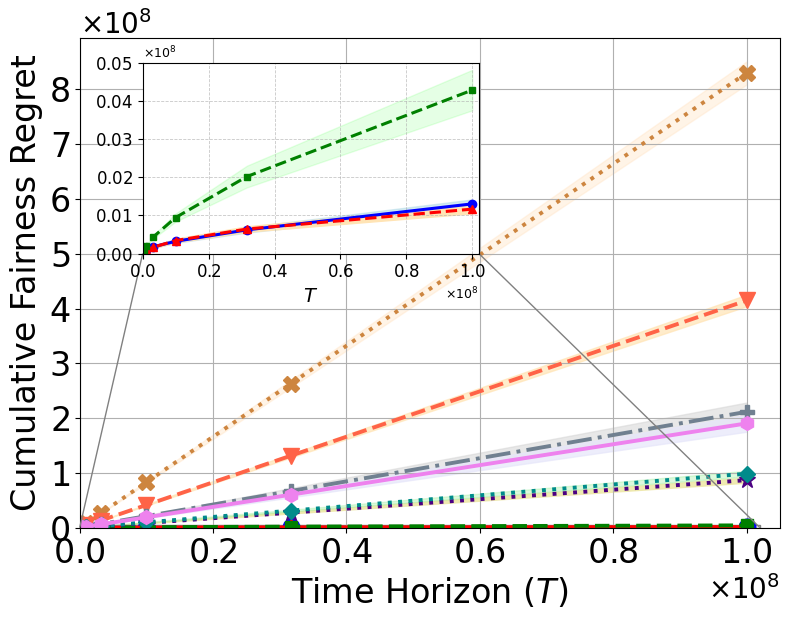}
        \caption{Cumulative Fairness Regret}
        \label{fig:syn_linear}
    \end{subfigure}%
    \hfill
    \begin{subfigure}[t]{0.32\textwidth}
        \centering
        \includegraphics[width=\textwidth]{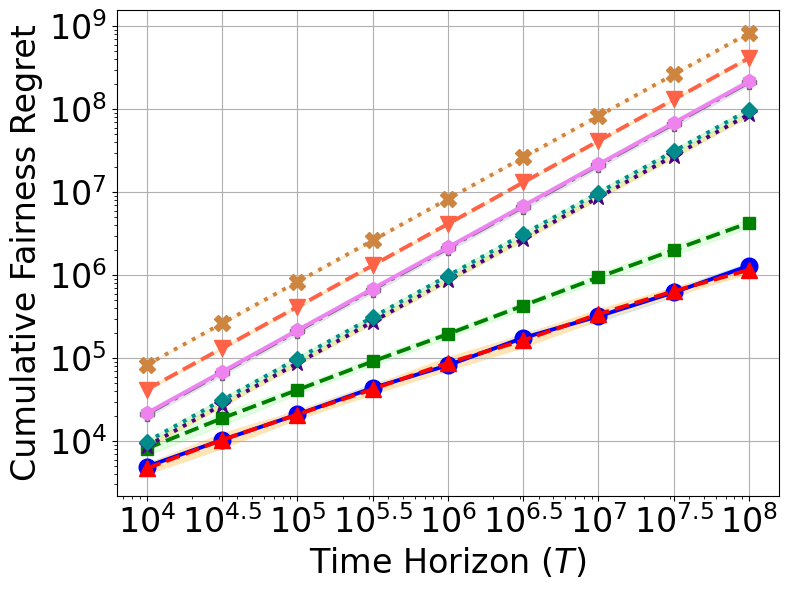}
        \caption{Cumulative Fairness Regret (log-log)}
        \label{fig:syn_log}
    \end{subfigure}
    \hfill
    \begin{subfigure}[t]{0.32\textwidth}
        \centering
        \includegraphics[width=\textwidth]{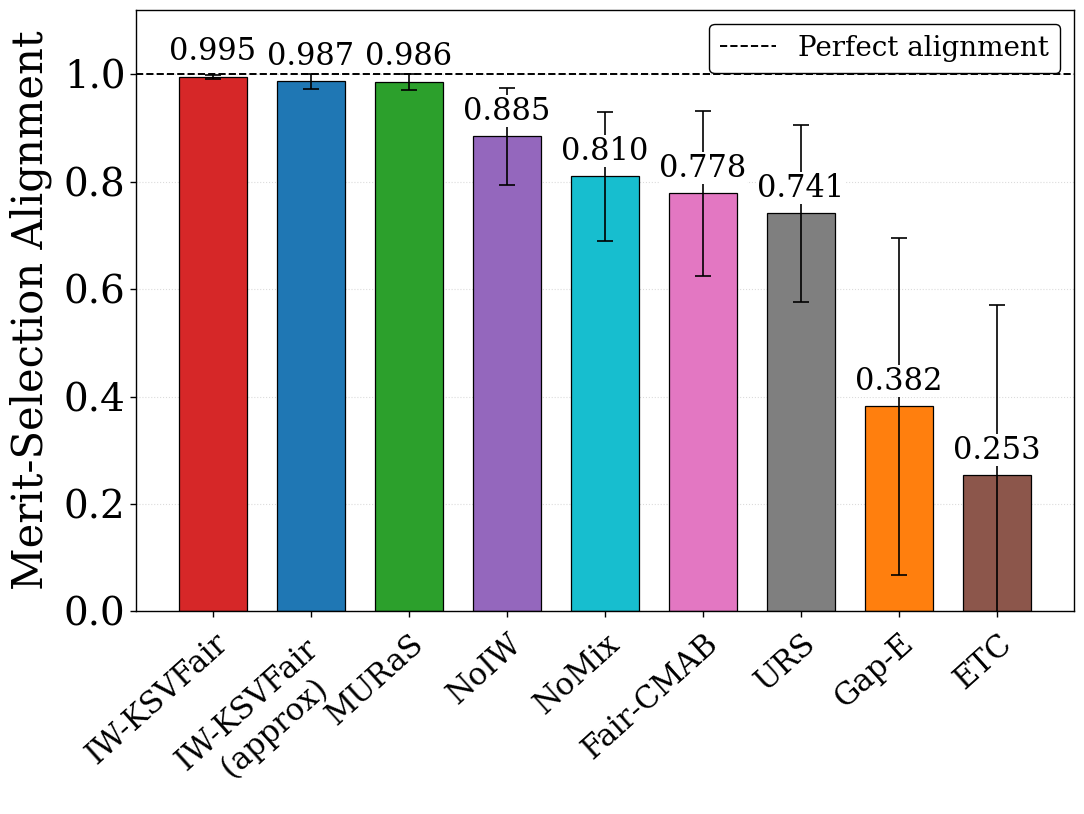}
        \caption{Merit-Selection Alignment at $T=10^8$}
        \label{fig:syn_ratio}
    \end{subfigure}%
    \caption{Comparative analysis of the considered algorithms on the Synthetic dataset.
    }
    \label{fig:exp_synthetic}
\end{figure*}

\begin{figure*}[h!]
    \centering
   
    \begin{subfigure}[t]{0.32\textwidth}
        \centering
        \includegraphics[width=\textwidth]{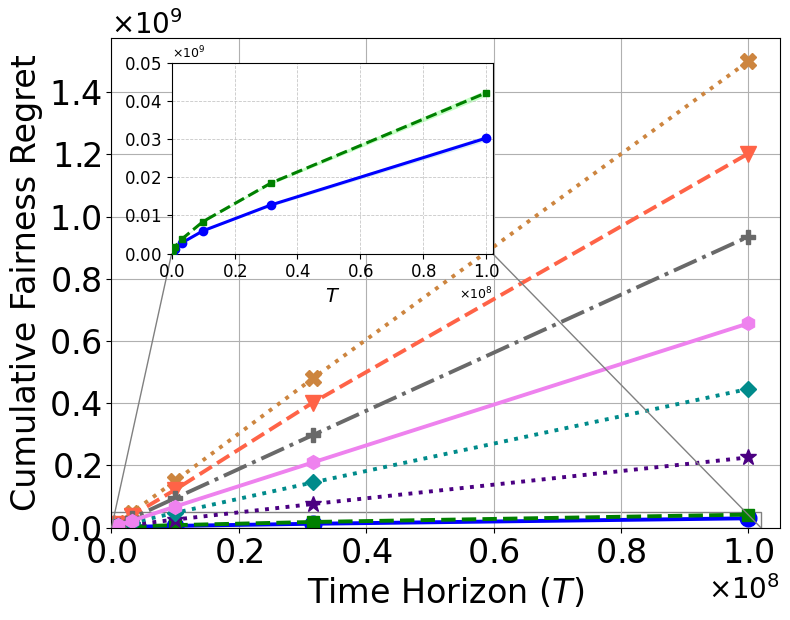}
        \caption{Cumulative Fairness Regret}
        \label{fig:sim_linear}
    \end{subfigure}%
    \hfill
     \begin{subfigure}[t]{0.32\textwidth}
        \centering
        \includegraphics[width=\textwidth]{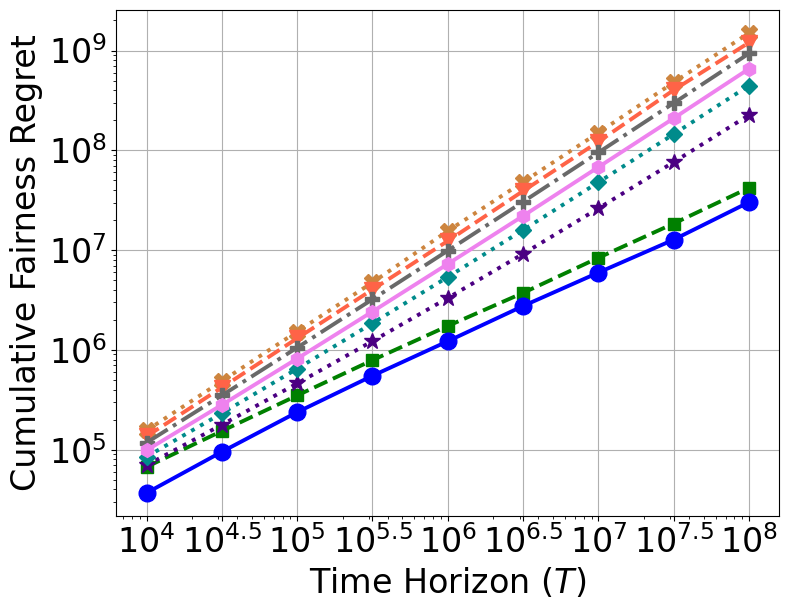}
        \caption{Cumulative Fairness Regret (log-log)}
        \label{fig:sim_log}
    \end{subfigure}
    \hfill 
    \begin{subfigure}[t]{0.32\textwidth}
        \centering
        \includegraphics[width=\textwidth]{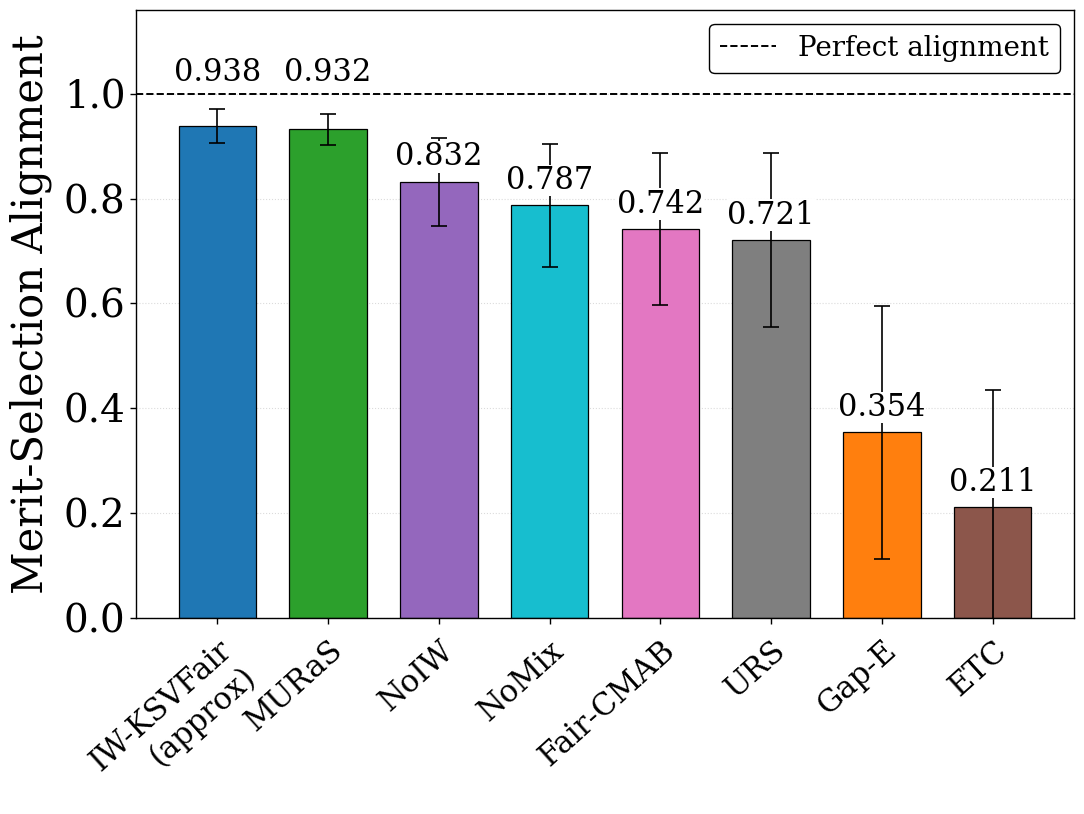}
        \caption{Merit-Selection Alignment at $T=10^8$
        }
        \label{fig:sim_ratio}
    \end{subfigure}%
    \caption{
    Comparative analysis of the considered algorithms on the Social Influence Maximization dataset.
    }
    \label{fig:exp_sim}

    \label{fig:exp3_sim}
\end{figure*}

\subsection{Datasets}

\noindent\textbf{Synthetic dataset.}~
We consider $M=20$ and $K=5$.
For each arm $a\in[M]$, its mean quality is sampled independently as
$
\mu_a \sim \mathcal{U}([0.2,0.9]).
$
Along the lines of \cite{etc},
when a set $S$ is queried at time $t$, the realized quality of each selected arm is
$
X_{a,t} = \mu_a + \varepsilon_{a,t},
$
where $\varepsilon_{a,t}$ is independently sampled from a zero-mean Gaussian distribution with standard deviation $0.1$, truncated to 
$[-0.1,0.1]$.
The observed full-bandit reward is
$
\widehat{V}_t(S)
=
1-\prod_{a\in S}(1-X_{a,t})
$.
\\[-0.5em]

\noindent\textbf{Social Influence Maximization (SIM).}~
We use a subgraph of the Facebook ego-network dataset~\cite{maindataset},
containing 
$M=534$ nodes and $8158$ directed edges. 
We consider $K=10$.
Influence spreads according to the Independent Cascade model, with activation probability $0.1$ for every edge; this is along the lines of \cite{etc}. 
A full-bandit observation $\widehat V_t(S)$ is the fraction of nodes influenced by seed set $S$.
This setting evaluates whether the algorithms selects seed nodes in proportion to the their marginal contributions to influence spread under full-bandit feedback.

\subsection{Experimental Setup}
We evaluate all algorithms on the following values of $T$ (time horizons):
$\{10^x: x = 4,4.5, \ldots, 7.5, 8\}$.
Every experiment is repeated using ten random seeds:
$\{1,2,\ldots,10\}$.
Our experiments
are conducted using Python 3.11  
on a workstation with Intel Xeon w7-2495X CPU @ 2.50 GHz, 64 GB RAM, and NVIDIA RTX PRO 4500 Blackwell GPU with 32 GB VRAM. 
Every evaluation of a coalition value is counted as one full-bandit query. 
Under our experimental setup,
for a single instance at time horizon $T=10^8$, MURaS and \ouralgo\ (approx) each required approximately $80$ seconds on the synthetic dataset, while exact IW-KSVFair required approximately $75000$ seconds ($\approx 21$ hours). On SIM, MURaS and \ouralgo\ (approx) required approximately $200$ and $1500$ seconds, respectively. Hence, the scalable approximation of \ouralgo\ substantially reduces the computational cost compared with exact set enumeration, although it remains more computationally expensive than MURaS because the set distribution needs to be calibrated repeatedly.

In all our experiments, we consider $\epsilon=10^{-4}$.
Based on our theoretical analysis, in order to achieve the said regret bounds, we require $R=\widetilde{\Theta}(T^{2/3})$ for MURaS and $\eta=\Theta(\frac{M^{2/3}}{K^{1/3}T^{1/3}})$ for \ouralgo. 
In our experiments, we set  
$R=T^{2/3}/\Lambda$
and
$\eta=\frac{M^{2/3}}{\Gamma K^{1/3}T^{1/3}}$,
where 
$2MR$ is the number of full-bandit queries performed by MURaS during
exploration, 
and $\eta T$ is the expected number of full-bandit queries corresponding to uniform exploration in \ouralgo.
We use $\Lambda=8, \Gamma=0.8$ for synthetic dataset, and $\Lambda=100, \Gamma=3$ for SIM dataset.
An ablation study over different values of $\Lambda$ and $\Gamma$ for synthetic dataset is presented in Appendix~\ref{appendix: gamma}.

Moreover, constructing the distribution \(P_t\) in Algorithm~\ref{algo:iw_ksvfair} (see line~\ref{line:8}) exactly is computationally expensive, as it is defined over all \(\binom{M}{K}\) feasible \(K\)-sets and must be recalibrated at every policy update so that its induced arm marginals match \(\pi_t\). This computation is feasible for the synthetic setting with \(M=20\) and \(K=5\), but becomes computationally expensive for the SIM setting with \(M=534\) and \(K=10\). We therefore use \ouralgo~(approx), based on conditional Poisson sampling and dynamic programming. The details are provided in Appendix~\ref{appendix:entropy}

\subsection{Experimental Results}

All reported results are averaged over ten random seeds. The shaded regions in the line plots and the error bars in the bar plots represent standard deviation. For the synthetic dataset, the reference policy $\pi^*$ is computed from the exact $K$-Shapley values. For SIM, it is computed using Monte Carlo estimation.
We observed that $\pi^* \in \mathcal{S}_{M,K}$ for both synthetic and SIM datasets, and its detailed computation is provided in Appendix~\ref{appendix:fair_policy}.
The learning algorithms do not have access to this reference policy.
\\[-0.5em]

\noindent\textbf{\textit{Cumulative Fairness Regret:}} 
We first evaluate the cumulative query-normalized fairness regret with respect to $\pi^*$, which selects each arm in proportion to
its true $K$-Shapley value.
From Figure~\ref{fig:exp_synthetic},
we observe that \ouralgo\ (approx) closely matches \ouralgo\ on the synthetic dataset, showing that the scalable approximation of the set distribution does not result in a substantial loss in fairness.
Figures~\ref{fig:syn_linear} and~\ref{fig:sim_linear} show that they achieve the lowest cumulative fairness regret across the considered time horizons on both datasets. 
Note that \ouralgo, its scalable approximation, and MURaS operate at a substantially lower regret scale than the remaining baselines, and so, are barely distinguishable in the full plots; the zoomed-in views highlight their relative performance.
The log-log plots in Figures~\ref{fig:syn_log} and \ref{fig:sim_log}
further show that the slopes for \ouralgo\ and MURaS are close to $2/3$, consistent with their theoretical
$\widetilde{O}(T^{2/3})$ guarantees. In contrast, the remaining algorithms exhibit slopes close to $1$, indicating linear fairness regret.

MURaS also exhibits sublinear regret, but incurs substantially higher cumulative regret because its uniform exploration phase does not adapt to the estimated merits of the arms. The substantially larger regret of NoIW and NoMix demonstrates the importance of both importance-weighted estimation and uniform mixing, respectively. The remaining baselines accumulate considerably higher regret because their selection policies are not designed to match the target $K$-Shapley-proportional marginals.
Their per-round policy mismatch persists even at large time horizons, causing fairness regret to accumulate at a nearly constant rate, resulting in linear regret. Additional results on the cumulative fairness regret of \ouralgo\ and MURaS for different values of $M$ and $K$ are presented in Appendix~\ref{appendix:abalition}.
\\[-0.5em]

\noindent\textbf{\textit{Merit-Selection Alignment:}}
We evaluate how closely the empirical selection frequencies (number of times each arm is pulled) of an algorithm aligns with the arms' merits (normalized $K$-Shapley values). Let $m_a$ denote the normalized merit of arm $a$ and $s_a$ its normalized empirical selection frequency. For each arm $a$, we define the merit-selection alignment score as
$
b_a
=
\min\left\{\frac{m_a}{s_a},\frac{s_a}{m_a}\right\}$,
with $b_a=0$ when $s_a=0$.
The score lies in $[0,1]$, where $b_a=1$ indicates that arm $a$ is selected exactly in proportion to its merit. For each algorithm, we report the mean alignment score
$
B=\frac{1}{M}\sum_{a=1}^{M} b_a.
$
$B = 1$ indicates exact merit-proportional selection.

At time horizon $T=10^8$, Figures~\ref{fig:syn_ratio} and~\ref{fig:sim_ratio} show that \ouralgo\ achieves nearly perfect merit-selection alignment, which is closely matched by its scalable approximation.
MURaS also attains high terminal alignment after committing to its estimated fair policy.
However, its considerably larger cumulative fairness regret shows that terminal alignment alone does not capture the unfairness incurred during its uniform exploration phase. 
NoIW and NoMix obtain lower alignment, demonstrating the importance of unbiased estimation and sufficient exploration, respectively.
URS and Fair-CMAB provide exposure without adapting it to heterogeneous merits, whereas GAP-E and ETCG concentrate primarily on highly ranked arms and therefore exhibit considerably weaker merit-proportional alignment.
An armwise comparison of normalized merit and empirical selection frequency is provided in Appendix~\ref{appendix:merit_sel_freq}.

\section{Conclusion}

In this work, we introduced a novel framework for ensuring meritocratic fairness in CMAB under full-bandit feedback. Unlike existing fairness models that rely on pre-defined quotas or exposure-based measures, our approach uses the Shapley value to estimate each arm’s contribution and selects arms in proportion to their merit. We proposed the \ouralgo\ algorithm, which adaptively estimates Shapley values while mitigating the effect of noise from stochastic feedback. Our theoretical analysis showed that \ouralgo\ achieves sublinear fairness regret, ensuring that fairness improves consistently with time. 
Our experiments show that \ouralgo\ learns $K$-Shapley value-proportional selection marginals substantially earlier than uniform explore-then-commit learning. Moreover, its scalable approximation closely matches it without a substantial loss in fairness.
In future work, an important theoretical direction is to extend our regret guarantees to the practically relevant setting in which the distribution $P_t$ is computed only approximately; currently, we validate the effect of this approximation only empirically. We also plan to extend the framework to dynamic and non-stationary environments, where arms may arrive or depart and their merits may evolve over time.

\ifCLASSOPTIONcompsoc
\else
\fi

\ifCLASSOPTIONcaptionsoff
  \newpage
\fi

\bibliographystyle{IEEEtran}
\bibliography{references}

\newcounter{appsec}
\renewcommand{\theappsec}{\Alph{appsec}}

\newcommand{\appsection}[2]{%
  \refstepcounter{appsec}%
  \section*{\theappsec. #1}%
  \label{#2}%
}

\newcounter{appsubsec}[appsec]
\renewcommand{\theappsubsec}{\theappsec.\arabic{appsubsec}}

\newcommand{\appsubsection}[2]{%
  \refstepcounter{appsubsec}%
  \subsection*{\theappsubsec. #1}%
  \label{#2}%
}

\onecolumn
\appendix
\setcounter{appsec}{0}


\appsection{Proof of Theorem 1 (Properties and Uniqueness of $\bm{K}$-Shapley Value)}{appendix:proof_of_uniqueness}

\appsubsection{Axiomatic Properties}{}
\begin{axiom}[Symmetry]
Players $i$ and $j$ are said to be symmetric if 
$V(S \cup \{i\}) = V(S \cup \{j\})$, 
$ \forall\, S \subseteq [M] \setminus \{i,j\}, \ |S| \le K - 1.$ $K$-Shapley value is said to satisfy symmetry property if, for all pairs of symmetric players $i, j \in [M]$, $\phi_i^K = \phi_j^K$.
\end{axiom}

\noindent

\begin{axiom}[Linearity]
Let $([M], V_1, K)$ and $([M], V_2, K)$ be any two R-Games, and let $\alpha,\beta \in \mathbb{R}$. 
Define a new R-Game $([M],\, \alpha V_1 + \beta V_2, K)$ such that, for any coalition $S \subseteq [M]$ with $|S| \le K$,
$
(\alpha V_1 + \beta V_2)(S) = \alpha V_1(S) + \beta V_2(S).
$
The $K$-Shapley value satisfies the \textit{linearity} property if, $\forall i \in [M]$,
$
\phi_i^K(\alpha V_1 + \beta V_2) 
= \alpha\phi_i^K(V_1) + \beta\phi_i^K(V_2).
$
\end{axiom}

\begin{axiom}[Null Player]
A player $i \in [M]$ is called a null player if
$
V(S \cup \{i\}) = V(S), \forall\, S \subseteq [M] \setminus \{i\}, |S| \le K-1.
$ K-Shapley value satisfies the null player property if
$\phi_i^K = 0.$
\end{axiom}

For any base arm $i \in [M]$, let its marginal contribution to a coalition $S \subseteq [M] \setminus \{i\}$ of size $\le K-1$ be defined as $\Delta_i(S) = V(S \cup \{i\}) - V(S)$.
Let $\psi_i^K(S_k) = \sum_{S \subseteq S_k \setminus \{i\}}
\frac{|S|! \, (K - |S| - 1)!}{K!}\Delta_i(S)$ represent the average of marginal contributions of $i$ for a set $S_k$ across all possible coalitions $S\subseteq S_k\setminus \{i\}$. Then, the $K$-Shapley value $\phi_i^K$ denotes the expected value of $\psi_i^K(S_k)$ over all possible such sets $S_k$ of size $K$ and containing $i$, i.e.,
\begin{align*}
\phi_i^K = 
\sum_{S_k \subseteq [M] : |S_k| = K, i \in S_k}
\frac{1}{\binom{M - 1}{K - 1}}\psi_i^K(S_k)
\end{align*}

\begin{align*}
\phi_i^{K}
= 
\sum_{\substack{S_K \subseteq [M] : |S_K|=K,\, i \in S_K}}
\sum_{\substack{S \subseteq S_K \setminus \{i\}}}
\frac{|S|!\,(K-|S|-1)!}{\binom{M-1}{K-1}\,K!}
\big[V(S\cup\{i\}) - V(S)\big]
\end{align*}

\noindent We now show that $\phi_i^{K}$ satisfies the symmetry, 
linearity, and null player, K-efficiency properties.\\

\appsubsection{Proof of Symmetry Axiom}{appendix:symmetry}

If two arms $i, j \in [M]$ contribute equally to every possible subset of arms of size up to $K$,
then their $K$-Shapley values are identical, i.e.,
$$
\phi_i^K = \phi_j^K.
$$

Here, $M$ denotes the total number of arms, and $[M] := \{1,2,\dots,M\}$ denotes the corresponding
set of arms. Throughout this proof, coalitions are subsets of $[M]$.

Let the valuation function of any coalition $S \subseteq [M]$ be denoted by $V(S)$.
For an arm $i \in [M]$, its marginal contribution to a coalition
$S \subseteq [M] \setminus \{i\}$ is defined as

$$
\Delta_i(S) = V(S \cup \{i\}) - V(S).
$$

The $K$-Shapley value of arm $i$ is given by

$$
\phi_i^K =
\sum_{\substack{S_k \subseteq [M]: |S_k| = K,\\ i \in S_k}}
\;\sum_{S \subseteq S_k \setminus \{i\}}
\frac{|S|!\,(K - |S| - 1)!}{\binom{M-1}{K-1}K!}
\Delta_i(S).
$$

We aim to show that $\phi_i^K=\phi_j^K$ under the symmetry
condition
\[
V(S\cup\{i\})=V(S\cup\{j\}),
\qquad
\forall S\subseteq[M]\setminus\{i,j\},\ |S|\leq K-1.
\]

We show that every marginal contribution term appearing in the $K$-Shapley value of arm $i$
can be paired with a unique marginal contribution term for arm $j$ that has
(i) the same numerical value and
(ii) the same combinatorial weight.
Let $\mathcal{S}_i$ denote all possible subsets of $[M]$ that contain player $i$. 
We now consider all $K$-sized coalitions $S_k \in \mathcal{S}_i$ and analyze them case by case.

\textbf{Case 1:} $S_k$ contains $j$.
Consider any subset $S \subseteq S_k \setminus \{i\}$.
There are two subcases:

\textit{(a) Subcase 1a: $j \notin S$.}

In this case, $S \subseteq S_k \setminus \{i,j\}$.
The marginal contribution of arm $i$ to $S$ is

$$
\Delta_i(S) = V(S \cup \{i\}) - V(S).
$$

To compare with arm $j$, we consider the corresponding term in the $K$-Shapley expansion of
$\phi_j^K$, obtained by interchanging the roles of $i$ and $j$.
The marginal contribution of arm $j$ is

$$
\Delta_j(S) = V(S \cup \{j\}) - V(S).
$$

Since $S \subseteq [M] \setminus \{i,j\}$, the symmetry condition implies
$V(S \cup \{i\}) = V(S \cup \{j\})$.
Therefore,
$$
\Delta_i(S) = \Delta_j(S).
$$

Moreover, the combinatorial weight
$$
\frac{|S|!\,(K - |S| - 1)!}{\binom{M-1}{K-1}K!}
$$
depends only on the size $|S|$ and not on the identity of the arm.
Thus, the corresponding terms in $\phi_i^K$ and $\phi_j^K$ are identical.
\\

\textit{(b) Subcase 1b: $j \in S$.}

In this case, we can write

$$
S = S' \cup \{j\},
\qquad
\text{where } S' \subseteq S_k \setminus \{i,j\}.
$$

The marginal contribution of arm $i$ is

$$
\Delta_i(S' \cup \{j\})
= V(S' \cup \{i,j\}) - V(S' \cup \{j\}).
$$

To compare with arm $j$, we consider the corresponding subset
$S' \cup \{i\} \subseteq S_k \setminus \{j\}$.
The marginal contribution of arm $j$ is

$$
\Delta_j(S' \cup \{i\})
= V(S' \cup \{i,j\}) - V(S' \cup \{i\}).
$$

Since $S'\subseteq[M]\setminus\{i,j\}$, symmetry condition gives
\[
V(S'\cup\{i\})=V(S'\cup\{j\}).
\]

Moreover, $V(S'\cup\{i,j\})$ is common to both marginal
contributions. Therefore,
\[
\Delta_i(S'\cup\{j\})
=
\Delta_j(S'\cup\{i\}).
\]

The sizes of the subsets $S' \cup \{j\}$ and $S' \cup \{i\}$ are equal, and hence their
combinatorial weights in the $K$-Shapley formula are also identical.

Combining Subcases (a) and (b), we conclude that for every coalition $S_k$ containing both
$i$ and $j$, all marginal contribution terms in $\phi_i^K$ have matching terms in
$\phi_j^K$ with the same value and the same weight.

\bigskip

\textbf{Case 2:} $S_k$ does not contain $j$.

For any $S_k \in \mathcal{S}_i$, define the corresponding coalition

$$
S_k' := (S_k \setminus \{i\}) \cup \{j\}.
$$

Then $S_k'$ satisfies $|S_k'| = K$, $j \in S_k'$, and $i \notin S_k'$.

Now consider any subset $S \subseteq S_k \setminus \{i\}$.
The same subset $S$ also satisfies $S \subseteq S_k' \setminus \{j\}$.
The marginal contribution of arm $i$ to $S$ is

$$
\Delta_i(S) = V(S \cup \{i\}) - V(S),
$$

while the marginal contribution of arm $j$ to the corresponding subset is

$$
\Delta_j(S) = V(S \cup \{j\}) - V(S).
$$

Since $S \subseteq [M] \setminus \{i,j\}$, the symmetry condition implies
$V(S \cup \{i\}) = V(S \cup \{j\})$, and hence

$$
\Delta_i(S) = \Delta_j(S).
$$

The subset $S$ has the same size in both cases, so the combinatorial weight in the
$K$-Shapley formula is unchanged.
Thus, every term in $\phi_i^K$ arising from coalitions in $\mathcal{S}_i$ has a
corresponding term in $\phi_j^K$ with identical value and weight.

Across all coalitions $S_k \subseteq [M]$ of size $K$, every marginal contribution term
contributing to the $K$-Shapley value of arm $i$ has a unique corresponding term for arm $j$
with the same numerical value and the same combinatorial weight.
Therefore,

$$
\phi_i^K = \phi_j^K.
$$

\appsubsection{Linearity (Additivity) Axiom}{appendix:linearity}

Let $V_1$ and $V_2$ be two $K$-restricted valuation
functions defined for all $S\subseteq[M]$ with $|S|\leq K$.
For $\alpha,\beta\in\mathbb{R}$, define
\begin{align*}
V(S) = \alpha V_1(S) + \beta V_2(S). 
\end{align*}

We need to show that
\begin{align*}
\phi_i^{K}(V) = \alpha\phi_i^{K}(V_1) + \beta\phi_i^{K}(V_2) .
\end{align*}

From the definition of $K$-Shapley Value,
\begin{align*}
\phi_i^{K}(V)
&= 
\sum_{S_K}\sum_S
\frac{|S|!\,(K-|S|-1)!}{\binom{M-1}{K-1}\,K!}
\big[V(S\cup\{i\}) - V(S)\big]
\end{align*}
Using $V(S)=\alpha V_1(S)+\beta V_2(S)$ gives
\begin{align*}
V(S\cup\{i\}) - V(S)
= \alpha[V_1(S\cup\{i\}) - V_1(S)]\\
+ \beta[V_2(S\cup\{i\}) - V_2(S)].
\end{align*}

Now, substituting this into the summation and splitting the sum,
\begin{align*}
\phi_i^{K}(V)
&= 
\alpha\sum_{S_K}\sum_S
\frac{|S|!\,(K-|S|-1)!}{\binom{M-1}{K-1}\,K!}
[V_1(S\cup\{i\}) - V_1(S)] \\
&\quad + \beta\sum_{S_K}\sum_S
\frac{|S|!\,(K-|S|-1)!}{\binom{M-1}{K-1}\,K!}
[V_2(S\cup\{i\}) - V_2(S)].
\end{align*}

Each summation above is exactly $\phi_i^{K}(V_1)$ or $\phi_i^{K}(V_2)$.
Therefore,
\begin{align*}
\phi_i^{K}(V) = \alpha\phi_i^{K}(V_1) + \beta\phi_i^{K}(V_2).
\end{align*}

\appsubsection{Null Player Axiom}{appendix:null}

Suppose arm $i$ is a null player, i.e.,
$
V(S\cup\{i\})=V(S),
\forall S\subseteq[M]\setminus\{i\}
\text{ with } |S|\leq K-1.
$
Then $V(S\cup\{i\}) - V(S) = 0$ for all $S$, and hence
\begin{align*}
\phi_i^{K}
= 
\sum_{S_K}\sum_S
\frac{|S|!\,(K-|S|-1)!}{\binom{M-1}{K-1}\,K!}\times 0 = 0
\end{align*}
Since the $K$-Shapley value measures the average marginal contribution of each arm across all $K$-coalitions,
a player who adds no value has zero marginal contribution everywhere,
and therefore a zero average.\\

\appsubsection{K-Efficiency Axiom}{appendix:efficiency}

The $K$-Shapley value satisfies a modified notion of efficiency suitable for budgeted settings. 
Unlike the classical Shapley value, which distributes the total value of the grand coalition $[M]$ among all players, 
the $K$-Shapley value operates under a restricted setting where only coalitions of size at most $K$ are feasible. 
Its $K$-efficiency property ensures that the total contribution distributed across all $M$ players 
is consistent with the average valuation of all 
coalitions of size $K$.
Specifically, 
$$
\sum_{i \in [M]} \phi_i^{K}
=
\frac{1}{\binom{M-1}{K-1}}
\sum_{\substack{T_K \subseteq [M]\\ |T_K|=K}} V(T_K) .
$$

\noindent
To establish this condition, we first verify that the above equality holds directly from the definition of the $K$-Shapley value. 
Then, we interpret it in terms of expected coalition values to provide understanding of its meaning.

\bigskip
\noindent
Starting from the formulation of the $K$-Shapley value for player $i$:
$$
\phi_i^{K}
= 
\sum_{\substack{S_K \subseteq [M] : |S_K|=K,\, i \in S_K}}
\sum_{\substack{S \subseteq S_K \setminus \{i\}}}
\frac{|S|!\,(K-|S|-1)!}{\binom{M-1}{K-1}\,K!}
\big[V(S\cup\{i\}) - V(S)\big],
$$

Summing over all $M$ players:
$$
\sum_{i \in [M]} \phi_i^{K}
= 
\sum_{i \in [M]} 
\sum_{\substack{S_K \subseteq [M]: |S_K|=K,\, i \in S_K}}
\sum_{\substack{S \subseteq S_K \setminus \{i\}}}
\frac{|S|!\,(K-|S|-1)!}{\binom{M-1}{K-1}\,K!}
\big[V(S\cup\{i\}) - V(S)\big]
$$

Rearranging the order of summation, we group terms by each $K$-coalition $S_K$:
$$
\sum_{i \in [M]} \phi_i^{K}
=
\sum_{\substack{S_K \subseteq [M]: |S_K|=K}}
\sum_{i \in S_K}
\sum_{\substack{S \subseteq S_K \setminus \{i\}}}
\frac{|S|!\,(K-|S|-1)!}{\binom{M-1}{K-1}\,K!}
\big[V(S\cup\{i\}) - V(S)\big]
$$

\noindent
For each fixed coalition $S_K$, the inner double summation over $i$ and $S$ 
is exactly the efficiency decomposition of the standard Shapley value applied within that $K$-coalition, i.e.,
$$
\sum_{i \in S_K} \sum_{S \subseteq S_K \setminus \{i\}}
\frac{|S|!\,(K-|S|-1)!}{K!}[V(S\cup\{i\}) - V(S)] = V(S_K)
$$

Substituting this back, we get:
$$
\sum_{i \in [M]} \phi_i^{K}
=
\frac{1}{\binom{M-1}{K-1}}
\sum_{\substack{S_K \subseteq [M]\\ |S_K|=K}} V(S_K)
$$

Thus, the total of all players’ $K$-Shapley values equals a scaled aggregate 
of the valuations of all $K$-coalitions:
$$
\sum_{i \in [M]} \phi_i^{K}
=
\frac{1}{\binom{M-1}{K-1}}
\sum_{\substack{T_K \subseteq [M]\\ |T_K|=K}} V(T_K)
$$

\noindent
When $K = M$, we have $\binom{M-1}{M-1} = 1$, and hence the expression 
reduces to the standard Shapley efficiency condition:
$$
\sum_{i \in [M]} \phi_i^{M} = V([M])
$$

\noindent
Therefore, the $K$-Shapley value generalizes the classical efficiency property 
by ensuring that the total distributed contribution across all players 
is consistent with the expected coalition value under a $K$-restricted coalition framework.

\bigskip

\textbf{Part-II of the Theorem}

\appsubsection{Uniqueness of the $K$-Shapley Value}{appendix:unique}

We want to prove that,
for every $K$-restricted valuation function
$V$ defined on coalitions $S\subseteq[M]$ with $|S|\leq K$
and satisfying $V(\emptyset)=0$,
there exists a unique allocation function $\phi^K(V) = (\phi_1^K(V), \ldots, \phi_M^K(V))$ that satisfies the symmetry, linearity, null player and $K$-efficiency properties. The unique allocation satisfying these four properties is the $K$-Shapley value.\\

\noindent
\textbf{Definition :}
Let $D_K \subseteq M$ be a nonempty coalition with $|D_K| \le K$.
The corresponding $K$-restricted carrier game $([M], u_{D_K})$ is defined, for every feasible coalition 
$S \subseteq [M]$ with $|S| \le K$, as

$$
u_{D_K}(S) =
\begin{cases}
1, & \text{if } D_K \subseteq S,\\[4pt]
0, & \text{otherwise.}
\end{cases}
$$

\bigskip
\noindent\textbf{Claim 1:} Every $K$-restricted game $(M, V)$ is a linear combination of carrier games.

\begin{proof}
The space of $K$-restricted coalitional games over the set of players $M$ 
is a vector space of dimension 
$\displaystyle \sum_{k=1}^{K} \binom{M}{k}$, 
corresponding to all nonempty feasible coalitions of size at most $K$. 
The number of carrier games equals the number of such nonempty feasible coalitions, 
namely $\sum_{k=1}^{K} \binom{M}{k}$. 
To prove the above, it suffices to show that the carrier games are linearly independent 
over $\mathbb{R}^{\sum_{k=1}^{K} \binom{M}{k}}$; this will imply that they form a linear basis 
of the space of $K$-restricted games. 
Indeed, every set of $\sum_{k=1}^{K} \binom{M}{k}$ independent vectors in a vector space of the same dimension 
forms a basis for that space, and therefore every element of the vector space can be written as a linear combination 
of the basis elements.

Suppose, by contradiction, that the carrier games are linearly dependent. 
Then there exists a linear combination of carrier games with non-zero coefficients that sums to the zero vector. 
In other words, there exist real numbers $\alpha_{D_K}$, not all zero, such that
$$
\sum_{\substack{D_K \subseteq [M],\, D_K \neq \emptyset,\\ |D_K| \le K}} 
\alpha_{D_K} \, u_{D_K}(S) = 0,
\qquad \forall\, S \subseteq M \text{ with } |S| \le K.
$$

Let
$$
\mathcal{D} = \{\, D_K \subseteq [M] : D_K \neq \emptyset,\, |D_K| \le K,\, \alpha_{D_K} \neq 0 \,\}
$$

be the set of all feasible coalitions with non-zero coefficients in the linear combination above. 
Since we assumed that not all coefficients are zero, the set $\mathcal{D}$ is nonempty. 
Let $S_0 \in \mathcal{D}$ be a minimal coalition in $\mathcal{D}$; that is, there is no coalition in $\mathcal{D}$ 
strictly contained in $S_0$.

We will show that
$$
\sum_{\substack{D_K \subseteq [M],\, D_K \neq \emptyset,\\ |D_K| \le K}} 
\alpha_{D_K}\, u_{D_K}(S_0) \neq 0,
$$
in contradiction to the previous equation.

Note that
\begin{align*}
\sum_{\substack{D_K \subseteq [M],\, D_K \neq \emptyset,\\ |D_K| \le K}} 
\alpha_{D_K}\, u_{D_K}(S_0)
&=\sum_{\substack{D_K \subset S_0,\, D_K \neq \emptyset,\\ |D_K| \le K}} 
\alpha_{D_K}\, u_{D_K}(S_0)
+ \alpha_{S_0}\, u_{S_0}(S_0)
+ \sum_{\substack{D_K \nsubseteq S_0,\, |D_K| \le K}} 
\alpha_{D_K}\, u_{D_K}(S_0).
\end{align*}

We have $\alpha_{D_K} = 0$ for every ${D_K}$ satisfying ${D_K} \subset S_0$, 
since $S_0$ is a minimal coalition in $\mathcal{D}$. 
For every ${D_K}$ satisfying ${D_K} \nsubseteq S_0$, 
the definition of a carrier game implies that $u_{D_K}(S_0) = 0$. 
Therefore,
\begin{align*}
\sum_{\substack{D_K \subseteq [M],\, D_K \neq \emptyset,\\ |D_K| \le K}} 
\alpha_{D_K}\, u_{D_K}(S_0)
&= \alpha_{S_0}\, u_{S_0}(S_0)
= \alpha_{S_0}
\neq 0.
\end{align*}

This is the contradiction. Hence, the assumption that the carrier games are linearly dependent is false. This implies that the carrier games are linearly independent in the space of $K$-restricted coalitional games.
\end{proof}

\textbf{Claim 2:}
Let $D_K \subseteq [M]$ be a nonempty feasible coalition of size $|D_K| \leq K$, and let $\alpha \in \mathbb{R}$ be a constant. 
Define a $K$-restricted carrier game $([M], u_{D_K,\alpha})$ as follows. 
For every feasible coalition $S \subseteq M$ with $|S| \le K$,
$$
u_{D_K,\alpha}(S) =
\begin{cases}
\alpha, & \text{if } D_K \subseteq S, \\[4pt]
0, & \text{otherwise.}
\end{cases}
$$

If $\rho^K$ is a solution concept satisfying $K$-efficiency, symmetry, and the null player property, then
$$
\rho^K_i([M]; u_{D_K,\alpha}) =
\begin{cases}
\displaystyle  C_{|D_K|}\frac{\alpha}{|D_K|}, & \text{if } i \in D_K, \\[8pt]
0, & \text{if } i \notin D_K.
\end{cases}
$$
Here, $C_{|D_K|} = \frac{\dbinom{M-|D_K|}{K-|D_K|}}{\dbinom{M-1}{K-1}}$.

\begin{proof}
In the $K$-restricted game $([M]; u_{D_K,\alpha})$, 
every player $i \notin D_K$ is a null player, 
and every pair of players in $D_K$ are symmetric. 
By the $K$-efficiency property, the total contribution of these players must equal to
\begin{align*}
\sum_{i\in D_K} \rho^K_i([M]; u_{D_K,\alpha})
&= \frac{1}{\binom{M-1}{K-1}}\left(\sum_{D_K \subseteq T_K, |T_K|=K}V(T_K) + \sum_{D_K \nsubseteq T_K, |T_K|=K}V(T_K)\right)\\
&=\frac{1}{\binom{M-1}{K-1}}\sum_{D_K \subseteq T_K, |T_K|=K}\alpha\\
&=\frac{\binom{M-|D_K|}{K-|D_K|}}{\binom{M-1}{K-1}}\alpha
\end{align*}

By the symmetry property, all players in $D_K$ must receive equal value. 
Therefore, each player in $D_K$ receives $C_{|D_K|}\frac{\alpha}{|D_K|}$. 
For every $i\notin D_K$, the null player property gives
$
\rho_i^K([M];u_{D_K,\alpha})=0.
$
This proves the claim.
\end{proof}

\noindent
We have shown that the $K$-Shapley value $\mathrm{\phi}^K$ satisfies linearity, $K$-efficiency, symmetry, and the null player property. 
It remains to show that the $K$-Shapley value is the unique solution concept satisfying these properties.

Let $\rho^K$ therefore be a solution concept satisfying these properties; we will show that $\rho^K = \mathrm{\phi}^K$. 
Let $(M, V)$ be a $K$-restricted coalitional game. 
Claim~1 implies that $V$ can be expressed as a linear combination of carrier games of the form $u_{D_K,\alpha_{D_K}}$, 
corresponding to feasible coalitions $D_K$ of size at most~$K$. 
In other words, there exist real numbers $\alpha_{D_K}$ such that

$$
V(S) 
= 
\sum_{\substack{D_K \subseteq M,\, D_K \neq \emptyset,\\ |D_K| \le K}} 
u_{D_K,\alpha_{D_K}}(S),
\qquad \forall\, S \subseteq M \text{ with } |S| \le K.
$$

By Claim~2, since both $\rho^K$ and $\mathrm{\phi}^K$ satisfy symmetry, $K$-efficiency, and the null player property,

$$
\rho^K(M, u_{D_K,\alpha_{D_K}}) 
= 
\mathrm{\phi}^K(M, u_{D_K,\alpha_{D_K}}),
\qquad 
\forall\, D_K \subseteq M,\ D_K \neq \emptyset,\ |D_K| \le K.
$$

Since both $\rho^K$ and $\mathrm{\phi}^K$ satisfy linearity, we have
\begin{align*}
\rho^K(M, V)
&=
\sum_{\substack{D_K \subseteq M,\, D_K \neq \emptyset,\\ |D_K| \le K}} 
\rho^K(M, u_{D_K,\alpha_{D_K}}) \\
&=
\sum_{\substack{D_K \subseteq M,\, D_K \neq \emptyset,\\ |D_K| \le K}} 
\mathrm{\phi}^K(M, u_{D_K,\alpha_{D_K}}) \\
&=
\mathrm{\phi}^K(M, V).
\end{align*}

Because this holds for every $K$-restricted game $(M, V)$, 
we conclude that $\rho^K = \mathrm{\phi}^K$. 
In other words, every solution concept satisfying linearity, $K$-efficiency, symmetry, 
and the null player property is identical to the $K$-Shapley value.\\

\appsection{Proof of Lower Bound}{appendix:proof_of_lower_bound}
At each time step $t$, let $r_t=|S_t|\in\{1,\ldots,K\}$,
and let $p_t^{(r_t)}\in[0,1]^M$ denote the conditional
marginal vector of the learner's size-$r_t$ query, i.e., $p_t^{(r_t)}(i)
=
\Pr(i\in S_t\mid H_{t-1},r_t).$ Then, the query-normalized marginal vector is given as: $p_t:=\frac{K}{r_t}p_t^{(r_t)}.$ Similarly, if $\phi_i^K$ denote the $K$-Shapley value of arm $i$, then
the meritocratic fair marginal vector $p^\star$ as
\[
    p_i^\star
    =
    \frac{K\phi_i^K}
    {\sum_{j=1}^M \phi_j^K}.
\]
As a result, the query-normalized instantaneous fairness
loss satisfies
\[
\frac{K}{r_t}
\left\|
p_t^{(r_t)}-p^{\star(r_t)}
\right\|_1
=
\left\|
p_t-\pi^*
\right\|_1.
\]
Therefore, the query-normalized fairness regret can be
equivalently written as
\[
FR_T
=
\sum_{t=1}^T
\mathbb{E}
\left[
\left\|p_t-p^\star\right\|_1
\right].
\]
In the below proof, $p_t$ denotes this
query-normalized marginal vector.

\begin{proof}
At each round $t$, the learner selects a feasible set
$S_t\subseteq[M]$ with $|S_t|\le K$ and observes $Y_t 
\sim \operatorname{Bernoulli}(V(S_t))$.

The fairness regret is defined as
\[
    \mathrm{FR}_T
    =
    \sum_{t=1}^T
    \mathbb{E}
    \left[
        \left\|p_t-p^\star\right\|_1
    \right].
\]

We prove the result by constructing two hard instances.
Let $S^\dagger:=\{2,\ldots,K+1\},
C_{M,K}:=\binom{M-1}{K-1}.$
Define $
\beta:=\frac{1}{4K},
\zeta:=\frac{K-1}{M-1}\beta$. Consider any constant $0<\epsilon<\zeta$. Choose $0<\Delta\leq \frac{\zeta}{2},
\rho:=\frac{\Delta}{K C_{M,K}}.$ Define the instance $V_1$ by the modular valuation
\[
V_1(S)
=
\beta\mathbf{1}\{1\in S\}
+
(\zeta+\rho)
\sum_{i\in S^\dagger}\mathbf{1}\{i\in S\}
+
\zeta
\sum_{i=K+2}^{M}\mathbf{1}\{i\in S\},
\]
where the final sum is zero when $M=K+1$.
Define the base instance $V_0$ by $V_0(S)
=
V_1(S)
-
\Delta\mathbf{1}\{S^\dagger\subseteq S\}$. 
Since $|S^\dagger|=K$ and only coalitions of size at most $K$
are feasible, the two instances differ on exactly one feasible
coalition, namely $S^\dagger$.
The function $V_1$ is modular and hence monotone and
submodular. The perturbation in $V_0$ only reduces the
marginal gain obtained when completing $S^\dagger$, which
remains nonnegative since
$\zeta+\rho-\Delta\geq\zeta-\Delta\geq\zeta/2$.
Moreover, $0\leq V_0(S)\leq V_1(S)<1$ for every feasible
coalition $S$. Hence, both $V_0$ and $V_1$ are monotone, submodular, and take values between $0$ and $1$.
Further, under $V_0$ and $V_1$, we have $\phi_1^K(V_0) = \phi_1^K(V_1) = \beta$. Under $V_1$,
\[
\phi_i^K(V_1)
=
\begin{cases}
\zeta+\rho,
& i\in S^\dagger,\\[1mm]
\zeta,
& i\notin S^\dagger,\ i\neq 1.
\end{cases}
\]
For each
$i\in S^\dagger$, the interaction $-\Delta\mathbf{1}\{S^\dagger\subseteq S\}$
appears in exactly one of the $C_{M,K}$ size-$K$ coalitions
containing $i$. Within the coalition $S^\dagger$, its Shapley
contribution to each member is $-\Delta/K$. Therefore, its
contribution to the $K$-Shapley value of each
$i\in S^\dagger$ is $-\frac{\Delta}{K C_{M,K}}=-\rho.$ Thus,
\[
\phi_i^K(V_0)=\zeta,
\quad i=2,\ldots,M.
\]
Thus, the smallest $K$-Shapley value in either instance is
at least $\zeta>\epsilon$, and both instances satisfy the
condition $\phi_i^K\geq\epsilon$ for every $i\in[M]$.
Under $V_0$, the total merit is
\[
\Phi_0
=
\beta+(M-1)\zeta
=
K\beta.
\]
Hence,
\[
p_1^\star(V_0)=1,
\qquad
p_i^\star(V_0)
=
\frac{K\zeta}{K\beta}
=
\frac{K-1}{M-1},
\quad i=2,\ldots,M.
\]

Under $V_1$, the total merit is
\[
\Phi_1
=
K\beta+K\rho
=
K\beta+\frac{\Delta}{C_{M,K}},
\]
and hence
\[
p_1^\star(V_1)
=
\frac{K\beta}
{K\beta+\Delta/C_{M,K}}.
\]
For $i\in S^\dagger$,
\[
p_i^\star(V_1)
=
\frac{K(\zeta+\rho)}
{K\beta+\Delta/C_{M,K}}
=
\frac{\zeta+\rho}{\beta+\rho}
<1,
\]
and for $i\notin S^\dagger$, $i\neq1$,
\[
p_i^\star(V_1)
=
\frac{\zeta}{\beta+\rho}
<1.
\]
Therefore, $p^\star(V_1)\in\mathcal{S}_{M,K}$. 
Since both marginal vectors have total mass $K$, the net
change in coordinates $2,\ldots,M$ is equal in magnitude and
opposite in sign to the change in the first coordinate.
Therefore, by the triangle inequality,
\[
d = \left\|p^\star(V_1)-p^\star(V_0)\right\|_1
\geq
2\left|
p_1^\star(V_1)-p_1^\star(V_0)\right| = \frac{2\Delta}
{C_{M,K}K\beta+\Delta}.
\]

Since $\Delta\leq\zeta/2\leq C_{M,K}K\beta$, we obtain
\[
d
\geq
\frac{\Delta}{C_{M,K}K\beta}
=
\frac{4\Delta}{C_{M,K}}.
\]

Next, we bound the information available to distinguish
$V_0$ from $V_1$. Among feasible coalitions, the two
instances differ only on $S^\dagger$. Define
\[
N^\dagger(T)
:=
\sum_{t=1}^T
\mathbf{1}\{S_t=S^\dagger\}.
\]

Since $S^\dagger$ excludes arm
$1$ and has size $K$,
\begin{align*}
\mathbb{P}_0(S_t=S^\dagger\mid H_{t-1})
&\leq
\mathbb{P}_0(r_t=K,\ 1\notin S_t\mid H_{t-1})\\
&=
\mathbb{P}_0(r_t=K\mid H_{t-1})
\left(1-p_t^{(K)}(1)\right).
\end{align*}
Under $V_0$, $p_1^\star(V_0)=1$. Therefore,
\[
1-p_t^{(K)}(1)
=
\left|
p_t^{(K)}(1)-p_1^\star(V_0)
\right|
\leq
\left\|
p_t^{(K)}-p^\star(V_0)
\right\|_1.
\]
Thus,
\[
\mathbb{P}_0(S_t=S^\dagger\mid H_{t-1})
\leq
\mathbb{E}_0
\left[
\left\|p_t-p^\star(V_0)\right\|_1
\,\middle|\,
H_{t-1}
\right].
\]
Summing over $t$ gives
\[
\mathbb{E}_0[N^\dagger(T)]
\leq
\mathrm{FR}_T(V_0).
\]

Let $\mathbb{P}_0$ and $\mathbb{P}_1$ denote the distributions over the complete
interaction history under $V_0$ and $V_1$, respectively.
The chain rule for KL divergence gives
\[
\mathrm{KL}(\mathbb{P}_0,\mathbb{P}_1)
=
\sum_{t=1}^T
\mathbb{E}_{0}
\left[
\mathrm{KL}
\left(
\mathbb{P}_0(Y_t\mid H_{t-1},S_t)
\,\middle\|\,
\mathbb{P}_1(Y_t\mid H_{t-1},S_t)
\right)
\right].
\]

If $S_t\neq S^\dagger$, the two conditional observation
distributions are identical and their KL divergence is zero.
If $S_t=S^\dagger$, their means $(\mu_0 = V_0(S^\dagger), \mu_1 = V_1(S^\dagger))$ differ by $\Delta$. 

Using standard Bernoulli-KL upper bound
\[
\operatorname{kl}(p,q)
\leq
\frac{(p-q)^2}{q(1-q)}
\]
We get,
$\mathrm{kl}(\mu_0,\mu_1) \le \frac{\Delta^2}{\mu_1(1-\mu_1)}$
From construction, we know that $\mu_1(1-\mu_1) = K(\zeta+\rho)(1-K(\zeta+\rho)) \ge \frac{K-1}{4(M-1)}(1-K\zeta-\frac{\Delta}{C_{M,K}})\ge \frac{K-1}{4(M-1)}(1-K\zeta-\frac{\zeta}{2C_{M,K}})\ge \frac{K-1}{4(M-1)}(1-(K+\frac{1}{2})\zeta)$. Substituting value of $zeta$ and using the fact that $M > M-1\ge K$, we get $(1-\mu_1) \ge 3/4 $, resulting in $\mu_1(1-\mu_1) \ge \frac{3(K-1)}{16(M-1)}$.
Thus,
\[
\mathrm{kl}
\left( \mu_0,\mu_1
\right)
\leq
\frac{16\Delta^2(M-1)}{3(K-1)}.
\]
It follows that
\[
\mathrm{KL}(\mathbb{P}_0,\mathbb{P}_1)
\leq
\frac{16\Delta^2(M-1)}{3(K-1)}
\mathbb{E}_0[N^\dagger(T)].
\]

Consequently,
\[
\mathrm{KL}(\mathbb{P}_0,\mathbb{P}_1)
\leq
\frac{16\Delta^2(M-1)}{3(K-1)}
\mathrm{FR}_T(V_0).
\]

We now relate indistinguishability to fairness regret. For any marginal vector $p$,
\[
    \left\|p-p^\star(V_0)\right\|_1
    +
    \left\|p-p^\star(V_1)\right\|_1
    \ge
    \left\|p^\star(V_0)-p^\star(V_1)\right\|_1
    =
    d.
\]
For each round $t$, define the event
\[
    A_t
    =
    \left\{
        \left\|p_t-p^\star(V_0)\right\|_1
        \ge
        \frac d2
    \right\}.
\]
On $A_t^c$, the triangle inequality implies
\[
    \left\|p_t-p^\star(V_1)\right\|_1
    \ge
    \frac d2.
\]
Therefore,
\[
\begin{aligned}
    &\mathbb{E}_{\mathcal{H}_T\sim \mathbb{P}_0}
    \left[
        \left\|p_t-p^\star(V_0)\right\|_1
    \right]
    +
    \mathbb{E}_{\mathcal{H}_T\sim \mathbb{P}_1}
    \left[
        \left\|p_t-p^\star(V_1)\right\|_1
    \right] \\
    &\qquad\ge
    \frac d2 \mathbb{P}_0(A_t)
    +
    \frac d2 \mathbb{P}_1(A_t^c) \\
    &\qquad=
    \frac d2
    \left(
        \mathbb{P}_0(A_t)
        +
        1-\mathbb{P}_1(A_t)
    \right) \\
    &\qquad\ge
    \frac d2
    \left(
        1-\mathrm{TV}(\mathbb{P}_0,\mathbb{P}_1)
    \right),
\end{aligned}
\]
where $\mathrm{TV}(\cdot,\cdot)$ denotes total variation distance. Summing over $t=1,\ldots,T$
gives
\[
    \mathrm{FR}_T(V_0)+\mathrm{FR}_T(V_1)
    \ge
    \frac{Td}{2}
    \left(
        1-\mathrm{TV}(\mathbb{P}_0,\mathbb{P}_1)
    \right).
\]
By Pinsker's inequality, 
\[
    \mathrm{TV}(\mathbb{P}_0,\mathbb{P}_1)
    \le
    \sqrt{
        \frac{1}{2}
        \mathrm{KL}(\mathbb{P}_0,\mathbb{P}_1)
    }.
\]
We now choose
\[
\Delta=\alpha T^{-1/3},
\]
where $\alpha>0$ is a fixed constant. For sufficiently large
\[
T
\geq
\left(\frac{2\alpha}{\zeta}\right)^3,
\]
we have $\Delta\leq\zeta/2$, as required.
\\
Choose a constant $c_1>0$ such that
\[
\frac{16(M-1)\alpha^2c_1}{3(K-1)}
\leq
\frac12.
\]
There are two cases. First, suppose
$
\mathrm{FR}_T(V_0)
\geq
c_1T^{2/3}.
$
Then the desired result follows immediately.
\\
Otherwise,
\[
\mathrm{FR}_T(V_0)
<
c_1T^{2/3}.
\]
Using the KL bound,
\begin{align*}
\mathrm{KL}(\mathbb{P}_0,\mathbb{P}_1)
&\leq
\frac{16\Delta^2(M-1)}{3(K-1)}
\mathrm{FR}_T(V_0)\\
&<
\frac{16(M-1)\alpha^2T^{-2/3}}{3(K-1)}
c_1T^{2/3}\\
&=
\frac{16(M-1)\alpha^2c_1}{3(K-1)}\\
&\leq
\frac12.
\end{align*}
Pinsker's inequality therefore gives
\[
\operatorname{TV}(\mathbb{P}_0,\mathbb{P}_1)
\leq
\sqrt{\frac12\mathrm{KL}(\mathbb{P}_0,\mathbb{P}_1)}
\leq
\frac12.
\]
Using the testing inequality established above,
\[
\mathrm{FR}_T(V_0)+\mathrm{FR}_T(V_1)
\geq
\frac{Td}{2}
\left(1-\operatorname{TV}(\mathbb{P}_0,\mathbb{P}_1)\right)
\geq
\frac{Td}{4}.
\]
Since
$
d\geq\frac{4\Delta}{C_{M,K}},
$
we obtain
\[
\mathrm{FR}_T(V_0)+\mathrm{FR}_T(V_1)
\geq
\frac{T\Delta}{C_{M,K}}
=
\frac{\alpha}{C_{M,K}}T^{2/3}.
\]
Therefore,
\[
\max\left\{
\mathrm{FR}_T(V_0),
\mathrm{FR}_T(V_1)
\right\}
\geq
\frac{\alpha}{2C_{M,K}}T^{2/3}.
\]
\\
Combining the two cases yields
\[
\max\left\{
\mathrm{FR}_T(V_0),
\mathrm{FR}_T(V_1)
\right\}
\geq
\min\left\{
c_1,
\frac{\alpha}{2C_{M,K}}
\right\}
T^{2/3}.
\]
\end{proof}

\begin{table}[hb]
\centering
\caption{Notation table.}
\label{tab:notation}
\renewcommand{\arraystretch}{1.2}
\begin{tabular}{@{}p{0.10\linewidth}p{0.84\linewidth}@{}}
\toprule
\textbf{Notation} & \textbf{Description} \\
\midrule

$M$
&
Number of arms.
\\

$[M]$
&
Set of arms, $[M]=\{1,2,\ldots,M\}$.
\\

$K$
&
Maximum number of arms that can be selected in a coalition.
\\

$T$
&
Time horizon measured in full-bandit queries.
\\

$\mathcal{A}$
&
Set of feasible coalitions,
$\mathcal{A}=\{S\subseteq[M]:|S|\leq K\}$.
\\

$S_t$
&
Coalition evaluated at time $t$.
\\

$r_t$
&
Size of the coalition evaluated at time $t$, i.e.,
$r_t=|S_t|$.
\\

$V(S)$
&
Expected valuation of coalition $S$.
\\

$\widehat{V}(S)$
&
Noisy full-bandit observation obtained by evaluating coalition $S$.
\\

$\Delta_a(S)$
&
Marginal contribution of arm $a$ to coalition $S$, defined as
$\Delta_a(S)=V(S\cup\{a\})-V(S)$.
\\

$\phi_a^K$
&
True $K$-Shapley value of arm $a$, used as its merit.
\\

$N_{t,a}$
&
Number of marginal-contribution samples accumulated for arm $a$
before policy update $t$.
\\

$\mathcal{S}_{M,K}$
&
Set of feasible arm-marginal vectors,
$\mathcal{S}_{M,K}
=\{\boldsymbol{p}\in[0,1]^M:
\sum_{a=1}^{M}p(a)=K\}$.
\\

$\pi^{*}$
&
Target $K$-Shapley-proportional arm-marginal vector.
\\

$\pi_t$
&
Feasible arm-marginal vector constructed at policy update $t$.
\\

$P_t$
&
Adaptive distribution over size-$K$ coalitions whose arm
marginals are given by $\pi_t$.
\\

$P_t^{\mathrm{mix}}$
&
Mixture of $P_t$ with the uniform distribution over all
size-$K$ coalitions.
\\

$q_t(a)$
&
Marginal probability of selecting arm $a$ under
$P_t^{\mathrm{mix}}$.
\\

$\eta$
&
Uniform mixing parameter used in $P_t^{\mathrm{mix}}$.
\\

$W$
&
Upper bound on the importance weights, given by
$W=M/(K\eta)$.
\\

$w_t(a,S_t)$
&
Importance weight assigned to the observed marginal contribution
of arm $a$ when coalition $S_t$ is sampled.
\\

$Z_t(a)$
&
Importance-weighted marginal-contribution sample used to update
the estimate of $\phi_a^K$.
\\

$R$
&
Number of marginal-contribution samples collected per arm during
the exploration phase of MURaS.
\\

\bottomrule
\end{tabular}
\end{table}

\appsection{MURaS Algorithm and Regret Proof}{appendix:muras}

\begin{algorithm}[t]
\caption{\textbf{Meritocratic Uniform Random Sampling (MURaS)}}
\label{algo:URS+Fair}

\KwInput{$T$: timestamp budget, $M$: number of arms, $K$: number of arms selected per deployed round, $R$: number of Monte Carlo samples}
\KwOutput{Deployed selected sets $S_\tau$, marginal selection probabilities $\pi_\tau$, estimated merits $\{\hat{\phi}_a\}_{a\in[M]}$}

Initialize $\hat{\phi}_{a}\gets 0$, $N_a\gets 0$, and $C_a\gets 0$ for all $a\in[M]$\;

\For{$a\in[M]$}{
    \For{$r=1$ \KwTo $R$}{
        Uniformly at random select a subset
        $U_{a,r}\subseteq[M]\setminus\{a\}$ such that
        $|U_{a,r}|=K-1$\

        $P_{a,r}\gets U_{a,r}\cup\{a\}$
        \tcp*{$P_{a,r}$ is a set of size $K$ containing $a$}

        Sample a permutation $\sigma_{a,r}$ uniformly at random from all
        permutations of $P_{a,r}$\

        Let
        $B_a^{\sigma_{a,r}}
        \gets
        \{b\in P_{a,r}: b \text{ appears before } a
        \text{ in } \sigma_{a,r}\}$\

        $\Delta_{a,r}
        \gets
        \hat V\!\left(B_a^{\sigma_{(a,r)}}\cup\{a\}\right)
        -
        \hat V\!\left(B_a^{\sigma_{(a,r)}}\right)$\

        $\hat{\phi}_{a}
        \gets
        \hat{\phi}_{a}
        +
        \dfrac{1}{R}\Delta_{a,r}$\

        $N_a\gets N_a+1$\
    }
}

$\tau\gets 2MR$

\While{$\tau < T$}{
    $\tau\gets\tau+1$\

    \For{$a\in[M]$}{
        $\hat{\phi}_{a}
        \gets
        \max\{\hat{\phi}_{a},0\}$\
    }

    \eIf{$\sum_{j=1}^{M}\hat{\phi}_{j}=0$}{
        \For{$a\in[M]$}{
            $\bar{\pi}_{\tau}(a)\gets K/M$\
        }
    }{
        \For{$a\in[M]$}{
            $\bar{\pi}_{\tau}(a)
            \gets
            \dfrac{K\hat{\phi}_{a}}
            {\sum_{j=1}^{M}\hat{\phi}_{j}}$\
        }
    }

    $\pi_{\tau}
    \leftarrow
    \arg\min_{p\in[0,1]^M :
    \sum_{a=1}^{M}p(a)=K}
    \left\|p-\bar{\pi}_\tau\right\|_1
    $

    $S_{\tau}\sim\mathrm{RRS}(\pi_{\tau},K)$

    \textbf{Deploy selected set}\
}

\Return{$\{S_{\tau}\}_{\tau\geq 1}$,
$\{\pi_{\tau}\}_{\tau\geq 1}$,
$\{\hat{\phi}_a\}_{a\in[M]}$}\

\end{algorithm}

\appsubsection{MURaS: Meritocratic Uniform Random Sampling}{appendix:muras_algo}

We start with a naive approach, Uniform Random Sampling (URS), which approximates the Shapley values by sampling techniques. Once it has learnt the Shapley values accurately enough, it employs a meritocratic fairness policy on the learnt Shapley values. The details are described in Algorithm~\ref{algo:URS+Fair}. MURaS essentially learns the Shapley value estimates of each player via a $R$ randomly generated permutation.

Once the Shapley values are estimated, the arms are selected as per the policy given in the definition of Meritocratic fairness, on the basis of estimated Shapley values. 
Algorithm~\ref{algo:URS+Fair} ensures fairness in the following way. First, for the initial time steps, every arm has an equal chance of being selected due to uniform sampling, and once the Shapley values are closely approximated, the algorithm returns a policy that is meritocratic fair.  URS requires $\mathcal{O}(R\cdot M)$ to find the approximated Shapley values. This scales polynomially in $K$, making it tractable compared to the exponential 
cost of the exact Shapley value computation. 
For any feasible marginal vector
$p\in[0,1]^M$ satisfying $\sum_{a=1}^{M}p(a)=K$,
$\operatorname{RRS}(p,K)$ denotes a randomized rounding
scheme~\cite{RRS, delayedfeedback} that samples a set $S\subseteq[M]$ of size $K$ such
that
\[
\Pr(a\in S)=p(a),
\qquad \forall a\in[M].
\]

We now provide the proof of Theorem~\ref{thm:muras}.

\appsubsection{Regret Proof of MURaS}{appendix:muras_fairness_regret}

\begin{proof}
For each arm $a$, MURaS draws $R$ independent samples as follows. For each
$r\in[R]$, it samples a set $V_{a,r}$ uniformly from all $K$-subsets
containing $a$, then samples a uniformly random permutation $\sigma_r$ of
$V_{a,r}$. Let $B_a^{\sigma(a,r)}$ be the set of arms appearing before $a$ in
$\sigma_r$, and define

\[
    \tilde Z_r^a
    =
    \hat V(B_a^{(r)}\cup\{a\})-\hat V(B_a^{(r)}),
\]
where $\hat V(\cdot)$ is a noisy sample of the valuation. Define $\hat\phi_a
    =
    \frac1R\sum_{r=1}^R \tilde Z_r^a$.

Since $\tilde{Z}_r^a\in[-1,1]$ and
$\mathbb{E}[\tilde{Z}_r^a]=\phi_a^K$, Hoeffding's inequality and a union bound over
arms give, with probability at least $1-\delta_1$,
\[
    \left|
    \frac1R\sum_{r=1}^R \tilde Z_r^a-\phi_a^K
    \right|
    \le
    \sqrt{\frac{2}{R}\log\frac{2M}{\delta_1}}
    \qquad \forall a\in[M].
\]
Therefore, with probability at least $1-\delta_1$,
\[
    |\hat\phi_a-\phi_a^K|
    \le
    \epsilon
    =
    \sqrt{\frac{2}{R}\log\frac{2M}{\delta_1}}
    \qquad \forall a.
\]

Let $\Phi=\sum_{a=1}^M\phi_a^K$ and $\hat\Phi=\sum_{a=1}^M\hat\phi_a.$ 
Since Algorithm~2 clips the estimates to be non-negative,
and since $\phi_a^K\geq 0$, clipping does not increase the
estimation error. Moreover, 
since $\Phi \geq \gamma$, we have that
whenever $M\epsilon<\gamma$,
\[
\hat{\Phi}
\geq
\Phi-\sum_{a=1}^{M}
\left|\hat{\phi}_a-\phi_a^K\right|
\geq
\gamma-M\epsilon
>0.
\]
Therefore,

\begin{align*}
\sum_{a=1}^M |\pi^*(a)-\pi(a)| &\le \sum_{a=1}^M|{\pi}(a) - \tilde{\pi}(a)| + \sum_{a=1}^M|\pi^*(a) - \tilde{\pi}(a)| \le  2\sum_{a=1}^M|\pi^*(a)-\tilde\pi(a)|\\ 
&\le2K\sum_{a=1}^M
\left|
\frac{\hat\Phi \phi_a^K-\Phi\hat\phi_a}
{\Phi\hat\Phi}
\right| =
\frac{2K}{\Phi\hat\Phi}
\sum_{a=1}^M
\left|
\hat\Phi \phi_a^K-\Phi\hat\phi_a
\right|.
\end{align*}
Now add and subtract $\hat\Phi\hat\phi_a$ inside the numerator:
\begin{align*}
\hat\Phi \phi_a^K-\Phi\hat\phi_a
&=
\hat\Phi \phi_a^K
-
\hat\Phi\hat\phi_a
+
\hat\Phi\hat\phi_a
-
\Phi\hat\phi_a \\
&=
\hat\Phi(\phi_a^K-\hat\phi_a)
+
\hat\phi_a(\hat\Phi-\Phi).
\end{align*}
Therefore,\begin{align*}\sum_{a=1}^M |\pi^*(a)-\pi(a)|&\le\frac{2K}{\Phi\hat\Phi}\sum_{a=1}^M\left|\hat\Phi(\phi_a^K-\hat\phi_a)+\hat\phi_a(\hat\Phi-\Phi)\right| \\&\le\frac{2K}{\Phi\hat\Phi}\sum_{a=1}^M\left(\hat\Phi|\phi_a^K-\hat\phi_a|+\hat\phi_a|\hat\Phi-\Phi|\right) \\&=\frac{2K}{\Phi\hat\Phi}\left(\hat\Phi\sum_{a=1}^M|\phi_a^K-\hat\phi_a|+|\hat\Phi-\Phi|\sum_{a=1}^M\hat\phi_a\right).\end{align*}
Since $\sum_{a=1}^M\hat\phi_a=\hat\Phi$, we get\begin{align*}\sum_{a=1}^M |\pi^*(a)-\pi(a)|&\leq \frac{2K}{\Phi\hat\Phi}\left(\hat\Phi\sum_{a=1}^M|\phi_a^K-\hat\phi_a|+\hat\Phi|\hat\Phi-\Phi|\right) \\&=\frac{2K}{\Phi}\left(\sum_{a=1}^M|\phi_a^K-\hat\phi_a|+|\hat\Phi-\Phi|\right).\end{align*}Moreover, $|\hat\Phi-\Phi|=\left|\sum_{a=1}^M \hat\phi_a-\sum_{a=1}^M \phi_a^K\right| =\left|\sum_{a=1}^M(\hat\phi_a-\phi_a^K)\right| \le\sum_{a=1}^M|\hat\phi_a-\phi_a^K|$. Hence,\begin{align*}\sum_{a=1}^M |\pi^*(a)-\pi(a)|&\le\frac{4K}{\Phi}\sum_{a=1}^M|\phi_a^K-\hat\phi_a|.\end{align*}If $|\phi_a^K-\hat\phi_a|\le \epsilon$ for every $a\in[M]$, then\[    \sum_{a=1}^M|\phi_a^K-\hat\phi_a|    \le M\epsilon.\]Finally, using $\Phi\ge \gamma$, we obtain\[    \sum_{a=1}^M |\pi^*(a)-\pi(a)|    \le    \frac{4KM}{\gamma}\epsilon.\]

During exploration, MURaS spends $O(RM)$ valuation-sampling time steps. Since the
per-round fairness regret is at most $2K$, the exploration regret is
$O(KRM)$. For the remaining time steps, the regret per round is at most 
$\frac{4KM}{\gamma}\epsilon$.
Thus,
\[
    \mathrm{FR}_T
    \le
    O(KRM)
    +
    T\cdot
    \frac{4KM}{\gamma}
    \left[
    \sqrt{\frac{2}{R}\log\frac{2M}{\delta_1}}
    \right].
\]
Choosing
\[
    R=\tilde\Theta(T^{2/3})
\]
balances the exploration term $KRM$ with the estimation term
$T/\sqrt{R}$, yielding
\[
    \mathrm{FR}_T
    =
    \tilde O\left(\frac{KM}{\gamma}T^{2/3}\right).
\]
If $\gamma=\Omega(1)$, this becomes
\[
    \mathrm{FR}_T=\tilde O(KMT^{2/3}).
\]
\end{proof}

In the experiments, we set the number of marginal-contribution
samples collected per arm to
$
R=\left\lfloor\frac{T^{2/3}}{\Lambda}\right\rfloor,
$
where $\Lambda>0$ is a constant that controls the extent of exploration. 
Since each marginal-contribution sample requires two full-bandit
queries, MURaS uses $2MR$ queries during exploration. Therefore,
we choose $\Lambda$ such that
$
2MR\leq T.
$
Ignoring the floor operation, this condition is equivalent to
$
2 M \frac{T^{2/3}}{\Lambda}\leq T,
$
or equivalently
$\Lambda\geq\frac{2M}{T^{1/3}}.
$
To guarantee feasibility for every considered horizon
$T\geq 10^{4}$, we impose the horizon-independent
condition
$
\Lambda\geq \frac{2M}{10^{4/3}}.
$
The empirical
effect and selection of $\Lambda$ are discussed in
Appendix~\ref{appendix: gamma}.

Despite its simplicity, the MURaS approach has several limitations.  
For instance, 
by relying on URS, all arms are treated identically without considering their potential contribution. This uniform treatment leads to inefficient use of samples, as valuable arms do not receive more exploration, while less informative arms may be oversampled. Also, it does not exploit the sequential feedback structure of the bandit setting, where observations are limited to the chosen arms in each round, resulting in wasted information and higher sample complexity. These limitations highlight the need for a more adaptive, bandit-based approach that can leverage sequential feedback, allocate exploration more effectively, and provide Shapley-based merit estimates at lower cost. Regret is linear in $M$, therefore when $M$ is very high it leads to higher regret.

\appsection{Regret Proof of \ouralgo}{appendix:iwksvfair}
Let $[M]=\{1,\ldots,M\}$ be the set of arms. For any arm $a\in[M]$,
let $\phi_a^K$ denote its true $K$-Shapley value. Define
\[
    \Phi=\sum_{a=1}^M \phi_a^K
\]
We assume $\phi_a^K\in[0,1]\quad \forall a\in[M]$ and  the target meritocratic fair marginal vector $\pi^*(a)
    =
    \frac{K\phi_a^K}{\Phi} \in \mathcal{S}_{M,K}$. Let us denote  $
    t \in \{(\tau-1)K+1: \tau=1,\ldots,\left\lfloor\frac{T}{K}\right\rfloor\}
    $ as time steps denoting the starting of a macro-round in which \ouralgo\ pulls the arm as per policy $\pi_t$ derived from $\tilde\pi_t$. Here,  
\[
    \tilde\pi_t(a)
    =
    \frac{K\hat{\phi}_{t,a}}
    {\sum_{j=1}^M \hat{\phi}_{t,j}}.
\]
with $\hat{\phi}_{t,a}
    =
    \min\{\max\{\bar\phi_{t,a},\epsilon\},1\}$. Let $\mathcal{H}_{t-1}$ denote the history till which the last update happened. Let $P_t$ be an exactly computable distribution over $K$-sets with marginals
$\pi_t$. \ouralgo\ then defines the mixed set distribution
\[
    P_t^{\mathrm{mix}}(S)
    =
    (1-\eta)P_t(S)
    +
    \eta\frac{1}{\binom{M}{K}},
    \qquad |S|=K.
\]
The marginal probability that arm $a$ is selected under this mixed set
distribution is
\[
    q_t(a)
    =
    \mathbb P_{S\sim P_\tau^{\mathrm{mix}}}(a\in S)
\]
Since $P_t$ has marginals $\pi_t$, and uniform distribution over $K$-sets
has marginal $K/M$ for every arm, we get
\[
    q_t(a)
    =
    (1-\eta)\pi_t(a)
    +
    \eta\frac{K}{M}
\]

Note that for any arm $a$, $Z_t(a), N_{t,a}$ is updated atmost once in that macro-round. We prove the regret bound using the series of Lemma that we state and prove below:

\begin{lemma}
\label{lem:iwksvfair-unbiased}
Fix a macro-round $t$ and an arm $a$.  
Then, conditional on the event $a\in S_t$,
\[
    \mathbb E[Z_t(a)\mid \mathcal H_{t-1},a\in S_t]
    =
    \phi_a^K.
\]
Therefore, each time arm $a$ is selected and updated, $Z_t(a)$ is an unbiased
sample of the true $K$-Shapley value $\phi_a^K$.
\end{lemma}

\begin{proof}
Fix an arm $a$ and a round $t$. Since the algorithm updates arm $a$ only when
$a\in S_t$, the relevant sampling distribution is the conditional distribution
of $(S_t,\sigma_t)$ given $a\in S_t$.

The probability that $a$ is selected under $P_t^{\mathrm{mix}}$ is
\[
    q_t(a)
    =
    \sum_{S\ni a:\,|S|=K} P_t^{\mathrm{mix}}(S)
\]
Therefore, for any $K$ set $S$ with $a\in S$,
\[
    \mathbb P(S_t=S\mid a\in S_t)
    =
    \frac{P_t^{\mathrm{mix}}(S)}{q_t(a)}.
\]
After $S_t=S$ is selected, the algorithm samples a permutation uniformly from
$\Pi(S)$. Hence,
\[
    \mathbb P(\sigma_t=\sigma\mid S_t=S,a\in S_t)
    =
    \frac1{K!}
\]
Thus, the conditional probability of observing the pair $(S,\sigma)$ is
\[
    r_{t,a}(S,\sigma)
    =
    \mathbb P(S_t=S,\sigma_t=\sigma\mid a\in S_t)
    =
    \frac{P_t^{\mathrm{mix}}(S)}{q_t(a)}
    \frac1{K!}.
\]

The target $K$-Shapley distribution over pairs $(S,\sigma)$ with $a\in S$ is
uniform:
\[
    u_a(S,\sigma)
    =
    \frac{1}{\binom{M-1}{K-1}}
    \frac1{K!}
\]
Therefore, the importance weight for converting the conditional
sampling distribution $r_{t,a}$ into the target $K$-Shapley distribution
$u_a$ is
\[
    \frac{u_a(S,\sigma)}{r_{t,a}(S,\sigma)}
    =
    \frac{
        \frac{1}{\binom{M-1}{K-1}}\frac1{K!}
    }{
        \frac{P_t^{\mathrm{mix}}(S)}{q_t(a)}\frac1{K!}
    }
=
    \frac{q_t(a)}
    {\binom{M-1}{K-1}P_t^{\mathrm{mix}}(S)}.
\]
This is the weight used in IWKSVFair:
\[
    w_t(a,S_t)
    =
    \frac{q_t(a)}
    {\binom{M-1}{K-1}P_t^{\mathrm{mix}}(S_t)}
\]

\begin{align*}
\mathbb E[Z_t(a)\mid \mathcal H_{t-1},a\in S_t]
&=
\sum_{S\ni a:\,|S|=K}
\sum_{\sigma\in\Pi(S)}
r_{t,a}(S,\sigma)
w_t(a,S)
\Delta_a(S,\sigma) \\
&=
\sum_{S\ni a:\,|S|=K}
\sum_{\sigma\in\Pi(S)}
\frac{P_t^{\mathrm{mix}}(S)}{q_t(a)}
\frac1{K!}
\frac{q_t(a)}
{\binom{M-1}{K-1}P_t^{\mathrm{mix}}(S)}
\Delta_a(S,\sigma)
\end{align*}

\begin{align*}
\mathbb E[Z_t(a)\mid \mathcal H_{t-1},a\in S_t]
&=
\sum_{S\ni a:\,|S|=K}
\sum_{\sigma\in\Pi(S)}
\frac{1}{\binom{M-1}{K-1}}
\frac1{K!}
\Delta_a(S,\sigma) \\
&=
\phi_a^K
\qedhere
\end{align*}

\end{proof}

\begin{lemma}
\label{lem:iwksvfair-range-variance}
For every selected arm $a\in S_t$, the estimator
$
    Z_t(a)=w_t(a,S_t)\hat\Delta_t(a)
$
satisfies
$
    0\le |Z_t(a)|\le W,
$
where
$
    W=\frac{M}{K\eta}.
$
Moreover,
\[
    \operatorname{Var}(Z_t(a)\mid \mathcal H_{t-1},a\in S_t)
    \le W.
\]
\end{lemma}

\begin{proof}
By definition,
\[
    P_t^{\mathrm{mix}}(S)
    =
    (1-\eta)P_t(S)
    +
    \eta\frac1{\binom{M}{K}}
\]
Since $P_t(S)\ge 0$, we have
\[
    P_t^{\mathrm{mix}}(S)
    \ge
    \eta\frac1{\binom{M}{K}}
\]
Therefore,
\[
    \frac1{P_t^{\mathrm{mix}}(S)}
    \le
    \frac{\binom{M}{K}}{\eta}
\]
The importance weight is
\[
    w_t(a,S)
    =
    \frac{q_t(a)}
    {\binom{M-1}{K-1}P_t^{\mathrm{mix}}(S)}
\]
Since $q_t(a)\le 1$,
\begin{align*}
w_t(a,S)
&\le
\frac{1}
{\binom{M-1}{K-1}}
\frac{\binom{M}{K}}{\eta}
\end{align*}
Since
$
    \frac{\binom{M}{K}}{\binom{M-1}{K-1}}
    =
    \frac{M}{K}
$,
we get
\[
    w_t(a,S)
    \le
    \frac{M}{K\eta}
\]
Define
\[
    W=\frac{M}{K\eta}.
\]
Thus,
\[
    0\le w_t(a,S)\le W.
\]

Since $-1\le \hat\Delta_t(a)\le 1$, we get
\[
    0\le |Z_t(a)|
    =
    w_t(a,S_t) |\hat\Delta_t(a)|
    \le W
\]

Now we prove the variance bound. Since variance is at most the second moment,
\[
    \operatorname{Var}(Z_t(a)\mid \mathcal H_{t-1},a\in S_t)
    \le
    \mathbb E[Z_t(a)^2\mid \mathcal H_{t-1},a\in S_t]
\]
Using $0\le w_t(a,S_t)\le W$ and
$-1\le\hat\Delta_t(a)\le 1$, we have
\[
    Z_t(a)^2
    =
    w_t(a,S_t)^2\hat\Delta_t(a)^2
    \le
    Ww_t(a,S_t)\hat\Delta_t(a)^2
    \le
    Ww_t(a,S_t)
\]
Taking conditional expectation gives
\[
    \mathbb E[Z_t(a)^2\mid \mathcal H_{t-1},a\in S_t]
    \le
    W\,
    \mathbb E[w_t(a,S_t)\mid \mathcal H_{t-1},a\in S_t]
\]
Now we need to show that
\[
    \mathbb E[w_t(a,S_t)\mid \mathcal H_{t-1},a\in S_t]=1
\]

\begin{align*}
\mathbb E[w_t(a,S_t)\mid \mathcal H_{t-1},a\in S_t]
&=
\sum_{S\ni a:\,|S|=K}
\sum_{\sigma\in\Pi(S)}
\frac{P_t^{\mathrm{mix}}(S)}{q_t(a)}
\frac1{K!}
\frac{q_t(a)}
{\binom{M-1}{K-1}P_t^{\mathrm{mix}}(S)}
\\
&=
\sum_{S\ni a:\,|S|=K}
\sum_{\sigma\in\Pi(S)}
\frac{1}
{\binom{M-1}{K-1}}
\frac1{K!}
\end{align*}
There are exactly $\binom{M-1}{K-1}$ sets $S$ of size $K$ containing $a$,
and each such set has $K!$ permutations. Therefore,
\[
    \mathbb E[w_t(a,S_t)\mid \mathcal H_{t-1},a\in S_t]
    =
    \binom{M-1}{K-1}K!
    \frac{1}
    {\binom{M-1}{K-1}K!}
    =
    1
\]
Thus
\[
    \mathbb E[Z_t(a)^2\mid \mathcal H_{t-1},a\in S_t]
    \le W
\]
Therefore
\[
    \operatorname{Var}(Z_t(a)\mid \mathcal H_{t-1},a\in S_t)
    \le W 
    \qedhere
\]
\end{proof}

\begin{lemma}
\label{lemma:marginalvector}
Let $S_t$ be a random K-set with marginal vector $q_t$ and let $\sigma_t$ be a uniformly random permutation of $S_t$. For $r\in[K]$, let $P_{t,r}$ be the set containing the first $r$ elements of $\sigma_t$. Then $\mathbb{P}(i\in P_{t,r}|\mathcal{H}_{t-1}) = \frac{r}{K}q_t(i)$.
\end{lemma}
\begin{proof}
Conditioned on $i\in S_t$, we have $P(i\in P_{t,r}|\mathcal{H}_{t-1},\mathbf{1}(i \in S_t)) = \frac{r}{K}$. Thus, $\mathbb{P}(i\in P_{t,r}|\mathcal{H}_{t-1}) = \mathbb{E}[\mathbf{1}\{i\in S_t\}P(i\in P_{t,r}|\mathcal{H}_{t-1},S_t)]=\frac{r}{K}q_t(i)$.
\end{proof}

\begin{lemma}
\label{lemma:projection}
Let $\pi_t\in
\arg\min_{p\in\mathcal{S}_{M,K}}
\|p-\widetilde{\pi}_t\|_1.$
Then, for every $a\in[M]$, $\widetilde{\pi}_t(a)
\le
C_t\pi_t(a),$
where
$C_t
=
\max\left\{
1,\|\widetilde{\pi}_t\|_\infty
\right\}
\le K.$
\end{lemma}

\begin{proof}
Let
\[
H
:=
\left\{
a\in[M]:
\widetilde{\pi}_t(a)>1
\right\}
\]
denote the set of coordinates that violate the upper-bound
constraint, and let
\[
E
:=
\sum_{a\in H}
\left(
\widetilde{\pi}_t(a)-1
\right)
\]
denote their total excess mass.

For any $p\in\mathcal{S}_{M,K}$, since $p(a)\le1$ for every
$a\in H$, we have
\[
\widetilde{\pi}_t(a)-p(a)
\ge
\widetilde{\pi}_t(a)-1,
\qquad a\in H.
\]
Hence,
\[
\sum_{a=1}^M
\left(
\widetilde{\pi}_t(a)-p(a)
\right)_+
\ge
\sum_{a\in H}
\left(
\widetilde{\pi}_t(a)-p(a)
\right)
\ge
E.
\]

Moreover, $\widetilde{\pi}_t$ and $p$ have the same total
mass. Therefore,
\[
\sum_{a=1}^M
\left(
\widetilde{\pi}_t(a)-p(a)
\right)_+
=
\sum_{a=1}^M
\left(
p(a)-\widetilde{\pi}_t(a)
\right)_+.
\]
It follows that
\[
\begin{aligned}
\|\widetilde{\pi}_t-p\|_1
&=
\sum_{a=1}^M
\left(
\widetilde{\pi}_t(a)-p(a)
\right)_+
+
\sum_{a=1}^M
\left(
p(a)-\widetilde{\pi}_t(a)
\right)_+ \\
&=
2
\sum_{a=1}^M
\left(
\widetilde{\pi}_t(a)-p(a)
\right)_+ \\
&\ge 2E.
\end{aligned}
\]

We next show that this lower bound is attainable. Set
$p(a)=1$ for every $a\in H$ and redistribute the removed
mass $E$ among the coordinates in $H^c$, without decreasing
any of those coordinates. Such a redistribution is feasible
because the total unused capacity on $H^c$ is
\[
\begin{aligned}
\sum_{a\in H^c}
\left(
1-\widetilde{\pi}_t(a)
\right)
&=
M-|H|-\sum_{a\in H^c}\widetilde{\pi}_t(a) \\
&=
M-|H|-K+\sum_{a\in H}\widetilde{\pi}_t(a) \\
&=
M-K+E \\
&\ge E,
\end{aligned}
\]
where the last inequality follows from $M\ge K$.
For this feasible vector,
\[
\sum_{a\in H}
|\widetilde{\pi}_t(a)-p(a)|
=
E
\]
and
\[
\sum_{a\in H^c}
|\widetilde{\pi}_t(a)-p(a)|
=
E.
\]
Hence, the minimum $\ell_1$ distance is exactly $2E$. Now let $\pi_t$ be any $\ell_1$ projection. Since it attains
the lower bound $2E$, equality must hold in the inequalities
above. In particular,
\[
\pi_t(a)=1,
\qquad a\in H,
\]
and no mass can be removed from any coordinate in $H^c$.
Therefore,
\[
\pi_t(a)\ge\widetilde{\pi}_t(a),
\qquad a\in H^c.
\]
If $H=\phi$, it is easy to see that $\tilde\pi_t = \pi_t$.
If $H\ne \phi$, then $C_t\ge1$, therefore if $a\in H^c$,
\[
\widetilde{\pi}_t(a)
\le
\pi_t(a)
\le
C_t\pi_t(a).
\]
If $a\in H$, then $\pi_t(a)=1$, so
\[
\widetilde{\pi}_t(a)
\le
\|\widetilde{\pi}_t\|_\infty
\pi_t(a)
\le
C_t\pi_t(a).
\]
Thus,
\[
\widetilde{\pi}_t(a)
\le
C_t\pi_t(a),
\qquad \forall a\in[M].
\]

Finally, since $\widetilde{\pi}_t$ is nonnegative and has
total mass $K$,
\[
\|\widetilde{\pi}_t\|_\infty
\le
\sum_{a=1}^M\widetilde{\pi}_t(a)
=
K.
\]
Therefore, $C_t\le K$.
\end{proof}
\appsubsection{Regret Proof of~\ouralgo}{appendix:regret_iwksvfair}

\begin{proof}
\label{app:iwksv-regret}

We use $T$ to denote the total number of full-bandit queries.
Each macro-round consists of $K$ prefix queries. For ease of
presentation, suppose first that $T$ is divisible by $K$, and
define $n:=\frac{T}{K}$. We index macro-rounds by $\tau\in[n]$. The $r$-th prefix query of macro-round $\tau$ occurs at atomic query index
$t=(\tau-1)K+r, r\in[K]$.

For every arm $a\in[M]$, let $\Phi:=\sum_{a=1}^M\phi_a^K$, and assume $\phi_a^K\in[\epsilon,1], \forall a\in[M],$
for some $\epsilon>0$. The full-budget fair marginal vector
is $\pi^*(a)=\frac{K\phi_a^K}{\Phi}.$ We also assume $\pi^* \in \mathcal{S}_{M,K}$. At the beginning of macro-round $\tau$, let $N_{\tau,a}
=
\sum_{s=1}^{\tau-1}\mathbb{I}\{a\in S_s\}$
denote the number of importance-weighted observations
available for arm $a$. Let $\overline\phi_{\tau,a}$ be their
empirical average and $\widehat\phi_{\tau,a} = \min\{\max\{\overline\phi_{\tau,a},\epsilon\},1\}$ define the clipped estimate. Let
$\widehat\Phi_\tau
=
\sum_{a=1}^M\widehat\phi_{\tau,a}$ and 
$\widetilde\pi_\tau(a)
=
\frac{K\widehat\phi_{\tau,a}}{\widehat\Phi_\tau}$.
 The algorithm computes
$\pi_\tau
\in
\arg\min_{p\in\mathcal{S}_{M,K}}
\|p-\widetilde\pi_\tau\|_1.$ Let $C=\max_{\tau\in[n]}C_\tau$, with $C_\tau$ defined in Lemma \ref{lemma:projection}. The distribution actually used to sample $S_\tau$ has marginals $q_\tau(a)
=
(1-\eta)\pi_\tau(a)+\eta\frac{K}{M}.$
\\[-0.5em]

\noindent\textbf{Uniform concentration of the $K$-Shapley estimates.}~
Fix an arm $a$. Let $\tau_{a,j}$ be the macro-round in
which arm $a$ is selected for the $j$-th time, and define
$Z_{a,j}=Z_{\tau_{a,j}}(a),
X_{a,j}:=Z_{a,j}-\phi_a^K.$ Let $\mathcal{G}_{a,j-1}$ denote the history of all the rounds in which arm $a$ was selected. Then from Lemma \ref{lem:iwksvfair-unbiased} we have:
\[
\mathbb{E}[X_{a,j}\mid\mathcal{G}_{a,j-1}]=0.
\]
Furthermore, from Lemma \ref{lem:iwksvfair-range-variance},  $|Z_{a,j}|\le W$
and $\operatorname{Var}
\bigl(
Z_{a,j}\mid\mathcal{G}_{a,j-1}
\bigr)
\le W$. Since $\phi_a\in[0,1]$ and $W\ge1$,
$|X_{a,j}|
\le |Z_{a,j}|+\phi_a
\le W+1
\le 2W$ and $\mathbb{E}
\left[
X_{a,j}^2\mid\mathcal{G}_{a,j-1}
\right]
\le W$.

Let $L_\delta
:=
\log\left(
\frac{2M(n+1)}{\delta}
\right)$.
By the two-sided Freedman inequality, followed by a union
bound over all arms $a\in[M]$ and all sample counts
$j\in[n]$, with probability at least $1-\delta$, simultaneously
for every $a\in[M]$ and every macro-round $\tau\in[n+1]$,
\[
\left|
\overline\phi_{\tau,a}-\phi_a^K
\right|
\le
\sqrt{
\frac{2WL_\delta}{N_{\tau,a}\vee1}
}
+
\frac{2WL_\delta}
{3(N_{\tau,a}\vee1)}.
\]
Because $\phi_a\in[\epsilon,1]$ and projection onto
$[\epsilon,1]$ cannot increase the distance to any point in
that interval,
\[
e_{\tau,a} = \left|
\widehat\phi_{\tau,a}-\phi_a^K
\right|
\le
\left|
\overline\phi_{\tau,a}-\phi_a^K
\right| \le \sqrt{
\frac{2WL_\delta}{N_{\tau,a}\vee1}
}
+
\frac{2WL_\delta}
{3(N_{\tau,a}\vee1)} = b(N_{\tau,a}).
\]

\noindent\textbf{Relating policy error to estimation error.}~
By optimality of the $\ell_1$ projection and feasibility of
$\pi^*$, $\|\pi_\tau-\widetilde\pi_\tau\|_1
\le
\|\pi^*-\widetilde\pi_\tau\|_1$.
Hence, by the triangle inequality, $\|\pi_\tau-\pi^*\|_1
\le
2\|\widetilde\pi_\tau-\pi^*\|_1$.
Now,
\[
\begin{aligned}
\|\widetilde\pi_\tau-\pi^*\|_1
&=
K\sum_{a=1}^M
\left|
\frac{\widehat\phi_{\tau,a}}{\widehat\Phi_\tau}
-
\frac{\phi_a^K}{\Phi}
\right|\\
&\le
\frac{K}{\widehat\Phi_\tau}
\sum_{a=1}^M
|\widehat\phi_{\tau,a}-\phi_a^K|
+
K\sum_{a=1}^M
\phi_a^K
\left|
\frac{1}{\widehat\Phi_\tau}
-
\frac{1}{\Phi}
\right|\\
&=
\frac{K}{\widehat\Phi_\tau}
\sum_{a=1}^M e_{\tau,a}
+
\frac{K}{\widehat\Phi_\tau}
|\widehat\Phi_\tau-\Phi|\\
&\le
\frac{2K}{\widehat\Phi_\tau}
\sum_{a=1}^M e_{\tau,a}.
\end{aligned}
\]
It follows that
\[
\|\pi_\tau-\pi^*\|_1
\le
\frac{4K}{\widehat\Phi_\tau}
\sum_{a=1}^M e_{\tau,a}.
\]
Dividing and multiplying by $\widehat\phi_{\tau,a}$, we get:
\[
\begin{aligned}
\|\pi_t-\pi^*\|_1
&\le
4\sum_{a=1}^M
\widetilde\pi_t(a)
\frac{e_{t,a}}{\widehat\phi_{t,a}}\\
&\le
\frac{4}{\epsilon}
\sum_{a=1}^M
\widetilde\pi_t(a)e_{t,a}\\
&\le
\frac{4C}{\epsilon}
\sum_{a=1}^M
\pi_t(a)e_{t,a},
\end{aligned}
\]
where the final inequality follows from
Lemma~\ref{lemma:projection}.

Since
$
q_\tau=(1-\eta)\pi_\tau+\eta \pi_{unif},
$
we have
\[
\begin{aligned}
\|q_\tau-\pi^*\|_1
&\le
(1-\eta)\|\pi_\tau-\pi^*\|_1
+
\eta\|\pi_{unif}-\pi^*\|_1\\
&\le
\frac{4C(1-\eta)}{\epsilon}
\sum_{a=1}^M\pi_\tau(a)e_{\tau,a}
+
2\eta K.
\end{aligned}
\]
Moreover,
$
q_\tau(a)\ge(1-\eta)\pi_\tau(a)
$
for every arm. Therefore,
\begin{equation}
\label{eq:q-policy-error}
\|q_\tau-\pi^*\|_1
\le
\frac{4C}{\epsilon}
\sum_{a=1}^M q_\tau(a)e_{\tau,a}
+
2\eta K.
\end{equation}
\noindent\textbf{Query-normalized regret of a macro-round.}~
Let $P_{\tau,r}$ denote the prefix of size $r$ generated in
macro-round $t$. Its conditional marginal vector is $\pi_\tau^{(r)}
=
\frac{r}{K}q_\tau$ from Lemma \ref{lemma:marginalvector}, while the corresponding size-$r$ fair vector is
$\pi^{*,(r)}
=
\frac{r}{K}\pi^*$.
Therefore, its query-normalized fairness loss is
\[
\begin{aligned}
\frac{K}{r}
\left\|
\pi_\tau^{(r)}-\pi^{*,(r)}
\right\|_1
&=
\frac{K}{r}
\left\|
\frac{r}{K}q_\tau-\frac{r}{K}\pi^*
\right\|_1
&=
\|q_\tau-\pi^*\|_1.
\end{aligned}
\]
Since each macro-round contains exactly $K$ prefix queries,
its total query-normalized loss is
\begin{equation}
\label{eq:macro-regret-bound}
\mathcal{L}_t
=
K\|q_t-\pi^*\|_1 \le \frac{4CK}{\epsilon}
\sum_{a=1}^M q_t(a)e_{t,a}
+
2\eta K^2. 
\end{equation}

\noindent\textbf{Counting argument.}~
For every macro-round $\tau$ and arm $a$, define $I_{\tau,a}:=\mathbb{I}\{a\in S_\tau\}$.
Conditionally on the history before macro-round $\tau$, $\mathbb{E}[I_{\tau,a}\mid\mathcal{H}_{\tau-1}]
=
q_\tau(a)$. For a fixed arm $a$, every time $I_{\tau,a}=1$, the counter
$N_{\tau,a}$ increases by one. Therefore, 
\[
\begin{aligned}
\sum_{\tau=1}^n
\frac{I_{\tau,a}}
{\sqrt{N_{\tau,a}\vee1}}
&\le
1+\sum_{j=1}^{N_{n+1,a}-1}\frac{1}{\sqrt{j}}
&\le
1+2\sqrt{N_{n+1,a}},
\end{aligned}
\]
and
\[
\begin{aligned}
\sum_{\tau=1}^n
\frac{I_{\tau,a}}
{N_{\tau,a}\vee1}
&\le
1+\sum_{j=1}^{N_{n+1,a}-1}\frac{1}{j}
&\le
2+\log n.
\end{aligned}
\]

Summing the first inequality over the arms and applying
Cauchy--Schwarz gives
\[
\begin{aligned}
\sum_{a=1}^M\sqrt{N_{n+1,a}}
&\le
\sqrt{
M\sum_{a=1}^M N_{n+1,a}
}.
\end{aligned}
\]
Exactly $K$ arms are selected in each macro-round. Hence,
\[
\sum_{a=1}^M N_{n+1,a}
=
Kn
=
T.
\]
It follows that
\begin{equation}
\label{eq:sqrt-counting}
\sum_{\tau=1}^n\sum_{a=1}^M
\mathbb{E}
\left[
\frac{q_\tau(a)}
{\sqrt{N_{\tau,a}\vee1}}
\right]
\le
M+2\sqrt{MT}.
\end{equation}
Similarly,
\begin{equation}
\label{eq:linear-counting}
\sum_{\tau=1}^n\sum_{a=1}^M
\mathbb{E}
\left[
\frac{q_\tau(a)}
{N_{\tau,a}\vee1}
\right]
\le
M(2+\log n).
\end{equation}

On the Freedman event,
\[
e_{\tau,a}
\le
\sqrt{
\frac{2WL_\delta}{N_{\tau,a}\vee1}
}
+
\frac{2WL_\delta}
{3(N_{\tau,a}\vee1)}.
\]
Combining this with
\eqref{eq:sqrt-counting} and
\eqref{eq:linear-counting}, we obtain:
\begin{align}
\mathbb{E}
\left[
\sum_{t=1}^n\sum_{a=1}^M
q_t(a)e_{t,a}
\right]
&\quad\le
(1-\delta)\left[\sqrt{2WL_\delta}
\left(
M+2\sqrt{MT}
\right)
+
\frac{2WL_\delta}{3}
M(2+\log n)\right] + 2KT\delta .
\label{eq:weighted-error-count}
\end{align}

\noindent\textbf{Completing the regret bound.}~
Summing~\eqref{eq:macro-regret-bound} over the
$n=T/K$ macro-rounds and using
\eqref{eq:weighted-error-count}, 
\begin{align}
FR_T
\le{}&
(1-\delta)\left[\frac{4CK}{\epsilon}
\sqrt{2WL_\delta}
\left(
M+2\sqrt{MT}
\right)+
\frac{8CKWL_\delta}{3\epsilon}
M(2+\log n)
+
2\eta K^2n\right] + 2KT\delta.
\label{eq:regret-before-failure}
\end{align}
Thus, the total loss over $T$ queries is at most $2KT$ with probability atmost $\delta$.
\\
Choosing $\delta=T^{-1}$ gives
\begin{align}
FR_T
\le{}&
\frac{4CK}{\epsilon}
\sqrt{2WL_\delta}
\left(
M+2\sqrt{MT}
\right)+
\frac{8CKWL_\delta}{3\epsilon}
M(2+\log n)
+
2\eta KT
+
2K.
\label{eq:explicit-iwksv-bound}
\end{align}

Since $W\ge1$, the term proportional to $M\sqrt{W}$ is
absorbed by the term proportional to $MW$. Letting $\delta$ to be of the order $1/T$ and ignoring
logarithmic factors, we obtain
\begin{equation}
\label{eq:iwksv-general-C}
FR_T
=
\widetilde{O}
\left(
\frac{CK}{\epsilon}\sqrt{WMT}
+
\frac{CKWM}{\epsilon}
+
\eta KT
\right).
\end{equation}
Substituting
$
W=\frac{M}{K\eta}
$
into~\eqref{eq:iwksv-general-C} gives
\begin{equation}
\label{eq:iwksv-general-C-eta}
FR_T
=
\widetilde{O}
\left(
\frac{C}{\epsilon}M\sqrt{K}
\sqrt{\frac{T}{\eta}}
+
\frac{M^2C}{\epsilon\eta}
+
\eta KT
\right).
\end{equation}
Balancing the leading estimation term with the mixing term
gives
\[
\eta
= \Theta
\left(
\frac{M}
{\sqrt{KT}}
\right)^{2/3}.
\]
With this choice,
\begin{equation}
\label{eq:iwksv-general-C-final}
FR_T
=
\widetilde{O}
\left(
\frac{C}{\epsilon}M^{2/3}K^{2/3}
T^{2/3}
+
\frac{C}{\epsilon}M^{4/3}K^{1/3}
T^{1/3}
\right).
\end{equation}
This establishes a sublinear
$\widetilde{O}(T^{2/3})$ query-normalized fairness regret
bound.

If $T$ is not divisible by $K$, let
$n=\lfloor T/K\rfloor$ and run $n$ complete macro-rounds.
There are fewer than $K$ remaining atomic queries. Since the
query-normalized loss of each atomic query is at most $2K$,
the incomplete final block contributes at most $2K^2$.
Therefore, the same asymptotic bound continues to hold.

In experiments, we set the uniform mixing parameter to
$
\eta=\frac{M^{2/3}}{\Gamma K^{1/3}T^{1/3}},
$
where $\Gamma>0$ is a constant that controls the extent of uniform exploration.
A smaller value of $\Gamma$ results in a larger value of $\eta$ and
hence greater uniform exploration, while a larger
$\Gamma$ reduces the contribution of the uniform component.
As $\eta$ is a mixing probability, it must satisfy
$
\eta\leq 1.
$
So, we choose $\Gamma$ such that
$
\frac{M^{2/3}}{\Gamma K^{1/3}T^{1/3}}\leq 1,
$
or equivalently
$
\Gamma\geq
\frac{M^{2/3}}{K^{1/3}T^{1/3}}.
$
To guarantee feasibility for every considered horizon
$T\geq 10^{4}$, we impose the horizon-independent condition
$
\Gamma\geq
\frac{M^{2/3}}{K^{1/3}10^{4/3}}.
$
The empirical effect and selection of $\Gamma$ are discussed in
Appendix~\ref{appendix: gamma}.
\end{proof}

\appsubsection{Waterfilling}{appendix:poisson}
In our experiments, we use iterative capped proportional normalization, also known as upper-capped water filling~\cite{waterfilling}. Given non-negative $K$-Shapley scores $\hat\phi_i$, we initially assign marginal probabilities in proportion to these scores. Whenever the resulting marginal probability of an arm exceeds one, we set it to one, reduce the remaining budget by one, and redistribute the remaining probability mass among the unsaturated arms in proportion to their $K$-Shapley scores. This procedure is repeated until all marginal probabilities lie in $[0,1]$. Equivalently, when at least $K$ arms have positive merit, the resulting policy is $\pi_t(i)=\min\{1,\lambda_t \hat\phi_i\},$ where $\lambda_t \geq 0$ is chosen such that
$
\sum_{i\in[M]}\pi_t(i)=K.
$
If fewer than $K$ arms have positive scores, all positive-score arms are selected with probability one, and the remaining probability mass is distributed uniformly among the zero-score arms. Thus, the constructed marginal vector always satisfies $0\leq \pi_t(i)\leq 1$ for every $i\in[M]$ and
$
\sum_{i\in[M]}\pi_t(i)=K,
$
while preserving proportionality among all unsaturated positive-score arms.

Note that waterfilling is used only to construct feasible online
policies from the estimated $K$-Shapley values $\hat{\phi}_i$. The
optimal policy $\pi^*$ is instead computed directly from the true 
$K$-Shapley values as
\[
\pi^*(i)
=
\frac{K\phi_i^K}{\sum_{j\in[M]}\phi_j^K}.
\]
For every experimental instance in both the synthetic and SIM settings,
we verified that $\pi^*(i)\leq 1$ for all $i\in[M]$. Hence,
$\pi^*\in\mathcal{S}_{M,K}$ and no coordinate requires capping.

\appsubsection{Maximum-Entropy Sampling in \ouralgo}{appendix:entropy}

At each policy update, \ouralgo\ constructs a target marginal vector
$\boldsymbol{\pi}_t\in\mathcal{S}_{M,K}$. Since these arm-level
marginals do not uniquely determine a distribution over size-$K$ sets,
we use the maximum-entropy distribution~\cite{chen1994weighted, sampling}:
\[
P_{{\theta}}(S)
=
\frac{
\exp\!\left(\sum_{a\in S}\theta_a\right)
}{
Z_K({\theta})
},
\qquad |S|=K,
\]
where
\[
Z_K({\theta})
=
\sum_{\substack{S\subseteq[M]\\|S|=K}}
\exp\!\left(\sum_{a\in S}\theta_a\right).
\]
The parameter vector ${\theta}$ is calibrated so that the
resulting arm-inclusion probabilities match
$\pi_t$. This is the conditional Poisson representation of
the maximum-entropy distribution over fixed-cardinality sets. For the synthetic setting, where $M=20$ and $K=5$, we explicitly
enumerate all
$
\binom{20}{5}=15504
$
feasible sets and calibrate $\boldsymbol{\theta}$ using warm-started
BFGS (Broyden-Fletcher-Goldfarb-Shanno, an algorithm for unconstrained nonlinear optimization). The complete set probability vector can then be constructed and
sampled directly. Although this implementation provides an exact
maximum-entropy distribution up to numerical optimization tolerance,
its computational and memory requirements scale with
$\binom{M}{K}$ and therefore become prohibitive as $M$ or $K$
increases.

For SIM dataset where $M=534$ and $K=10$, explicit
enumeration is infeasible. So, we use a scalable
conditional Poisson implementation based on dynamic programming (DP). The
DP evaluates the partition function and the induced
marginals in $O(MK)$ time per calibration iteration and samples a
size-$K$ set without constructing the complete set-probability vector.
It also provides the probability $P_{\boldsymbol{\theta}}(S_t)$ of the
sampled set, which is required by the importance-weighted estimator.
The parameters are warm-started across policy updates and calibrated
until the maximum marginal error is below a prescribed 
tolerance.

Both implementations use the same conditional Poisson family. They
differ only in how its parameters, marginals, and samples are computed.
We refer to the dynamic-programming implementation as
\ouralgo~(approx) because the prescribed marginals are matched up to
the calibration tolerance rather than through explicit optimization
over all feasible sets. We also evaluate this approximation on the
synthetic setting to compare it with the enumerated implementation.

\begin{figure}[t]
    \centering
    \includegraphics[width=0.7\linewidth]{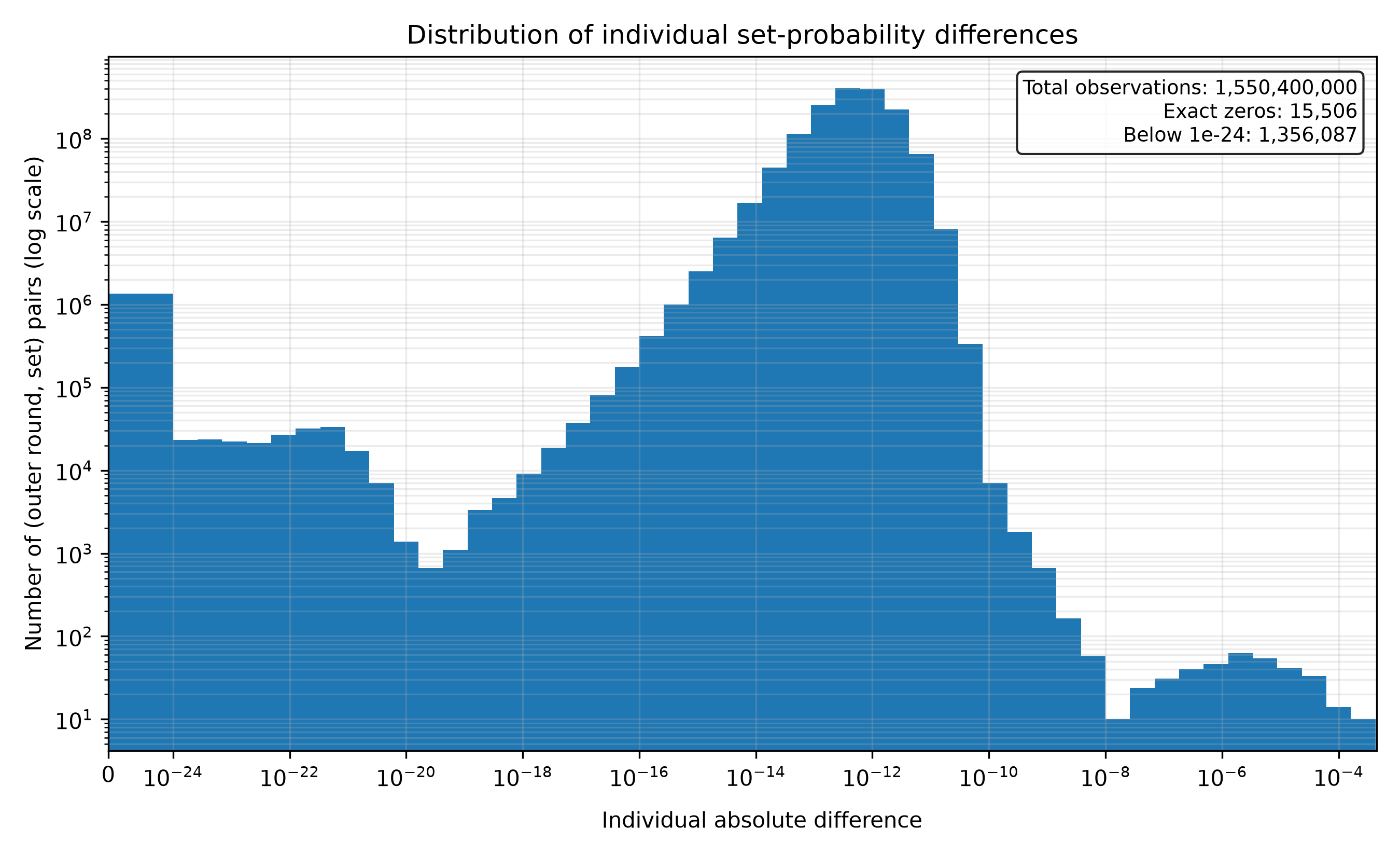}
    \caption{Distribution of the individual absolute differences
    $\left|
    P_t^{\mathrm{BFGS}}(S)
    -
    P_t^{\mathrm{DP}}(S)
    \right|$
    over all sets and all distribution-update rounds. Both axes are shown on
    logarithmic scales (differences below
$10^{-24}$ are grouped into the leftmost near-zero bin).}
    \label{fig:pt_individual_difference_histogram}
\end{figure}

Although the two implementations target the same distributional
family, they employ different numerical procedures. We therefore
empirically verify whether the scalable implementation reproduces the
distribution obtained by the enumerated reference implementation. This
comparison must be performed at the level of the complete subset
distribution, rather than only at the level of arm marginals. In
particular, different joint distributions can have identical arm
marginals while inducing substantially different dependencies among
the selected arms. Moreover, the subset probability enters directly
into the importance-weighted estimator.

For the synthetic dataset with $T=10^{6}$, we first run the exact
IW-KSVFair algorithm and record the sequence of feasible
arm-marginal vectors $\pi_t$ generated along its adaptive learning
trajectory. From this exact IW-KSVFair trajectory, we consider
$100{,}000$ evaluated policy-update rounds (outer rounds). At every such update,
the same target marginal vector $\pi_t$ is supplied to both
distribution construction procedures: the enumerated implementation,
whose parameters are calibrated using the
BFGS algorithm, and the scalable
DP implementation. 
Let $P_t^{\mathrm{BFGS}}$ and $P_t^{\mathrm{DP}}$ denote
the distributions obtained from the enumerated BFGS and
DP implementations, respectively, using the same
$\pi_t$ generated by exact IW-KSVFair. 
Since $M=20$ and $K=5$, each
distribution assigns probabilities to all
$
\binom{20}{5}=15{,}504
$
feasible size-$K$ sets. Hence, the comparison contains
$
100{,}000\times15{,}504
=
1{,}550{,}400{,}000
$
individual set-probability differences.
Figure~\ref{fig:pt_individual_difference_histogram} shows the distribution
of
\[
\left|
P_t^{\mathrm{BFGS}}(S)
-
P_t^{\mathrm{DP}}(S)
\right|
\]
over the evaluated policy updates and feasible sets. The histogram
is strongly concentrated at very small values, with most positive
differences lying approximately between $10^{-14}$ and $10^{-10}$.
Across all comparisons, the mean absolute difference is
$3.04\times10^{-10}$, the median among strictly positive differences
is $4.84\times10^{-13}$, and the maximum observed difference is
$2.98\times10^{-4}$. Also, $15{,}506$ comparisons have exactly
zero difference, $1{,}356{,}087$ strictly positive differences
are smaller than $10^{-24}$, and are grouped into the leftmost
near-zero bin. These results show that, for the generated marginal vectors, the scalable DP
implementation reproduces the enumerated BFGS set probabilities with
high pointwise accuracy.

\appsubsection{Computation of the Fair Policy ($\pi^*$)}{appendix:fair_policy}

The optimal fair policy $\pi^*$ represents the marginal selection probabilities that would be used if the $K$-Shapley values were known. For each arm $a$, the $K$-Shapley value can be written as
$$
\phi_a^K
=
\frac{1}{K}
\sum_{r=0}^{K-1}
\frac{1}{\binom{M-1}{r}}
\sum_{\substack{
B\subseteq[M]\setminus\{a\}\\
|B|=r
}}
\left[
V(B\cup\{a\})-V(B)
\right].
$$

This expression averages the marginal contribution of arm $a$ over predecessor coalitions of sizes $0,\ldots,K-1$. Every predecessor size receives weight $1/K$, and the predecessor coalition is uniformly distributed among all coalitions of that size.

For the synthetic environment, $M=20$ and $K=5$. The expected valuation of a set is
$$
V(S)
=
1-\prod_{a\in S}(1-\mu_a).
$$

For any arm $a\notin B$, its marginal contribution to $B$ is
$$
\begin{aligned}
\Delta_a(B)
&=
V(B\cup\{a\})-V(B)\\
&=
\left[
1-(1-\mu_a)\prod_{j\in B}(1-\mu_j)
\right]
-
\left[
1-\prod_{j\in B}(1-\mu_j)
\right]\\
&=
\mu_a
\prod_{j\in B}(1-\mu_j).
\end{aligned}
$$

We compute $\phi_a^K$ exactly by enumerating all predecessor coalitions $B\subseteq[M]\setminus\{a\}$ with $|B|\leq K-1$. The total number of predecessor coalitions for one arm is
$$
\sum_{r=0}^{K-1}\binom{M-1}{r}.
$$

For $M=20$ and $K=5$, this exact computation is practical and does not introduce Monte Carlo error.

\noindent\textbf{Monte Carlo Reference for Social Influence Maximization.}~
Exact $K$-Shapley computation is not practical for the Facebook graph because the number of coalitions is large and the valuation of a seed set is itself an expectation over random diffusion outcomes. So, we construct a fixed offline Monte Carlo reference.
For each node $a$, we independently repeat the following procedure $R_{\mathrm{ref}}=5000$ times:

\begin{itemize}
    \item Sample a set $S$ of size $K$ uniformly from all $K$-sets containing node $a$.
    \item Sample a uniformly random permutation of $S$.
    \item Let $B$ be the set of nodes that appear before $a$ in the permutation.
    \item Evaluate $V(B\cup\{a\})$ and $V(B)$ using the same live-edge realization.
    \item Record the marginal contribution
    $
    \Delta_a(B)
    =
    V(B\cup\{a\})-V(B).
    $
\end{itemize}

The reference estimate is
$$
\phi_a^{K,\mathrm{ref}}
=
\frac{1}{R_{\mathrm{ref}}}
\sum_{\ell=1}^{R_{\mathrm{ref}}}
\Delta_a^{(\ell)}.
$$

The same reference vector is used for every algorithm, time horizon, and online random seed. The reference is computed before the online experiments and is used only for evaluation. No learning algorithm receives access to this vector.

Given non-negative reference merits $\phi_a^{K,\mathrm{ref}}$, the direct proportional marginal vector is
$$
{\pi}^{\star}(a)
=
\frac{
K\phi_a^{K,\mathrm{ref}}
}{
\sum_{j=1}^{M}\phi_j^{K,\mathrm{ref}}
}.
$$
This vector sums to $K$ and 
\[
\max_{a} \pi^{\star}(a) \leq 1.
\]

\appsection{Ablation Study}
{appendix:abalition}

\begin{figure}[t]
    \centering

    \begin{subfigure}[t]{0.48\textwidth}
        \centering
        \includegraphics[width=\linewidth]{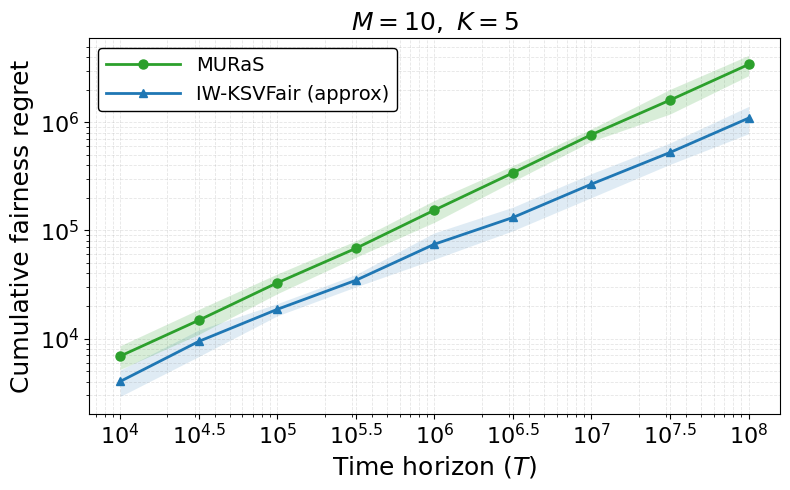}
        \subcaption{$M=10$, $K=5$}
        \label{fig:image1}
    \end{subfigure}
    \hfill
    \begin{subfigure}[t]{0.48\textwidth}
        \centering
        \includegraphics[width=\linewidth]{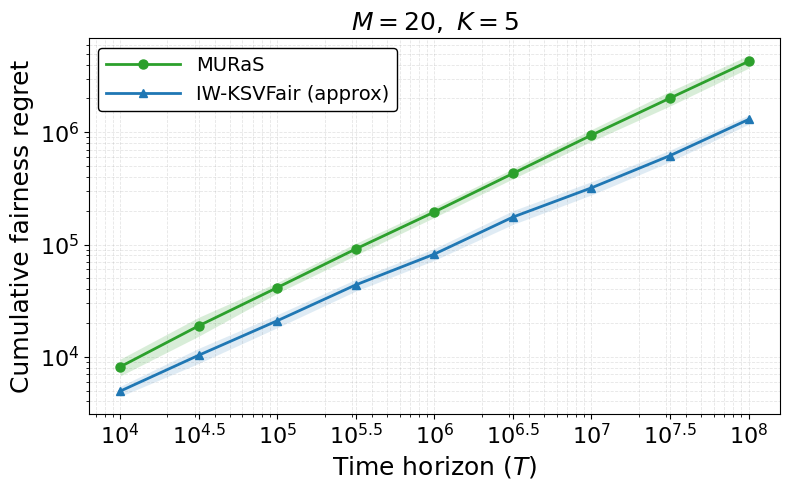}
        \subcaption{$M=20$, $K=5$}
        \label{fig:image2}
    \end{subfigure}

    \vspace{0.3cm}

    \begin{subfigure}[t]{0.48\textwidth}
        \centering
        \includegraphics[width=\linewidth]{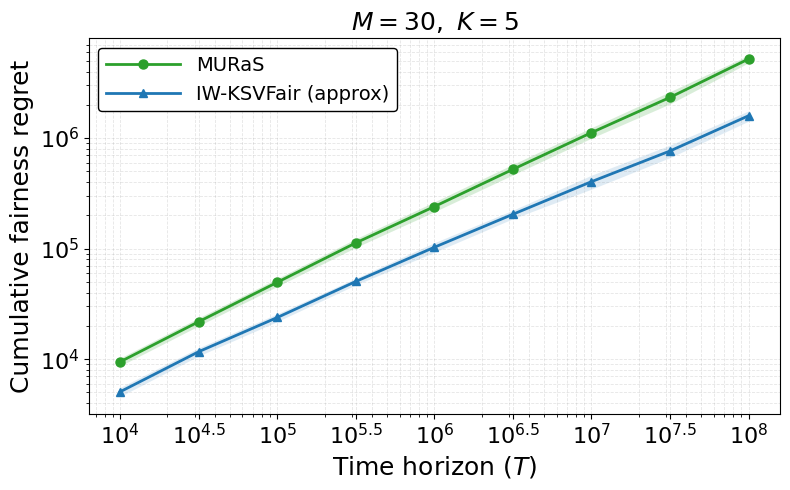}
        \subcaption{$M=30$, $K=5$}
        \label{fig:image3}
    \end{subfigure}
    \hfill
    \begin{subfigure}[t]{0.48\textwidth}
        \centering
        \includegraphics[width=\linewidth]{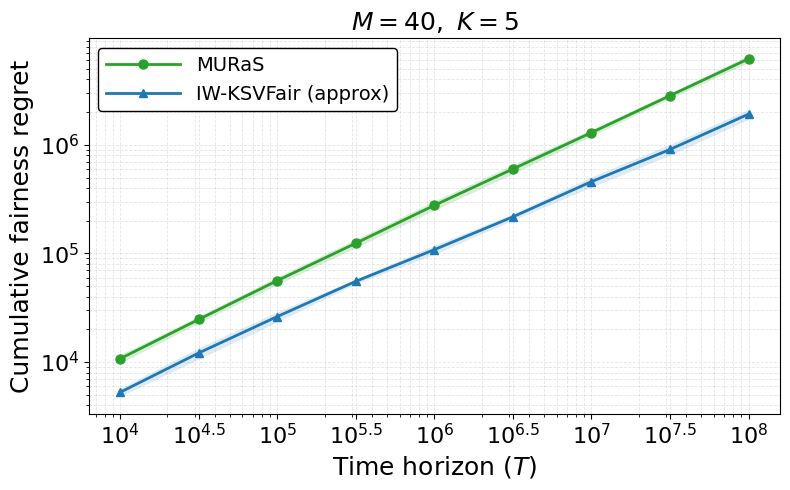}
        \subcaption{$M=40$, $K=5$}
        \label{fig:image4}
    \end{subfigure}

    \caption{Cumulative fairness regret of MURaS and \ouralgo~(approx) on synthetic dataset for $M\in\{10,20,30,40\}$ with $K=5$.}
    \label{fig:different_m}
\end{figure}

\appsubsection{Effect of the Number of Arms $M$}{appendix:different $M$}

Figure~\ref{fig:different_m} studies the effect of the number of arms on
the synthetic dataset by varying $M\in\{10,20,30,40\}$ while fixing the
selection budget at $K=5$. Across all values of $M$, both MURaS and \ouralgo~(approx)
exhibit sublinear growth in cumulative fairness regret. However,
\ouralgo~(approx) consistently achieves substantially lower regret
than MURaS over the entire time horizon. As $M$ increases, the cumulative fairness regret of both algorithms
increases because the learner must estimate the merits of a larger
number of arms and learn a fair policy over a larger action space. The
performance gap nevertheless persists across all considered values of
$M$. MURaS explores the arms uniformly before committing to its
estimated fair policy, and its exploration cost therefore increases
with the number of arms. In contrast, \ouralgo~(approx) continuously
adapts its sampling distribution using the estimated $K$-Shapley
values, resulting in considerably lower cumulative regret. These
results demonstrate that the empirical advantage of adaptive
merit-guided exploration is robust to an increasing number of arms.

\appsubsection{Effect of the Selection Budget $K$}{appendix:different_K}

Figure~\ref{fig:different_k} evaluates the effect of the selection
budget on the synthetic dataset by varying
$K\in\{3,5,10,15\}$ while fixing $M=20$. For every value of $K$,
both MURaS and \ouralgo~(approx) exhibit sublinear cumulative fairness
regret. Moreover, \ouralgo~(approx) consistently achieves lower regret
than MURaS.

The cumulative fairness regret of both algorithms increases with $K$,
which is consistent with the dependence on $K$ in their theoretical
regret bounds. However, the difference between the two algorithms
becomes smaller as $K$ increases. When $K$ is small, only a few arms
are selected in each round, and uniform exploration can differ
considerably from the merit-proportional target policy. Therefore,
MURaS incurs substantially higher regret during its uniform exploration
phase. In contrast, \ouralgo~(approx) adapts its selection probabilities
according to the estimated $K$-Shapley values.

When $K$ approaches $M$, most arms must be selected in every round, e.g., when $K=15$ and $M=20$, each round selects three-fourths
of the arms. As a result, both uniform exploration and the
merit-proportional policy assign high selection probabilities to most
arms. Hence, the two policies become more similar, which explains
why the regret gap between MURaS and \ouralgo~(approx) decreases for
larger values of $K$.

\begin{figure}[t]
    \centering

    \begin{subfigure}[t]{0.48\textwidth}
        \centering
        \includegraphics[width=\linewidth]{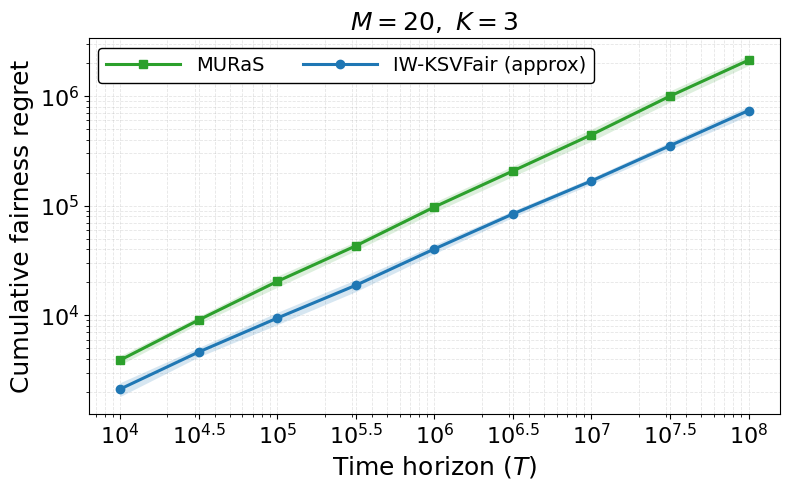}
        \subcaption{$M=20$, $K=3$}
        \label{fig:image1}
    \end{subfigure}
    \hfill
    \begin{subfigure}[t]{0.48\textwidth}
        \centering
        \includegraphics[width=\linewidth]{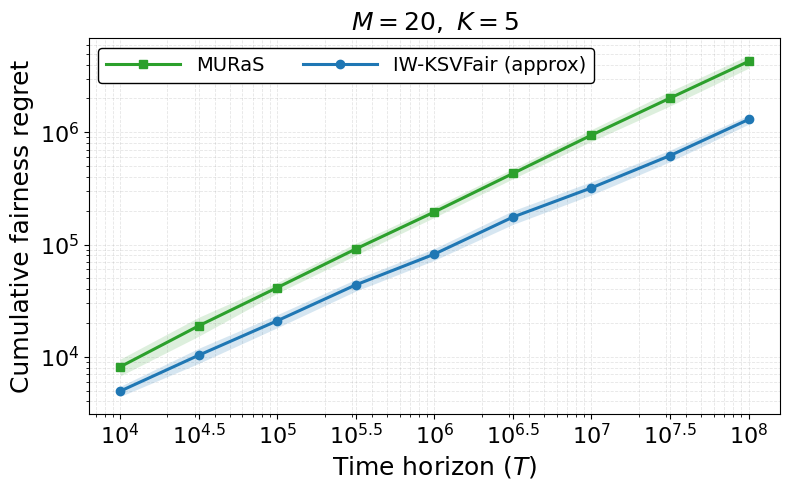}
        \subcaption{$M=20$, $K=5$}
        \label{fig:image2}
    \end{subfigure}

    \vspace{0.3cm}

    \begin{subfigure}[t]{0.48\textwidth}
        \centering
        \includegraphics[width=\linewidth]{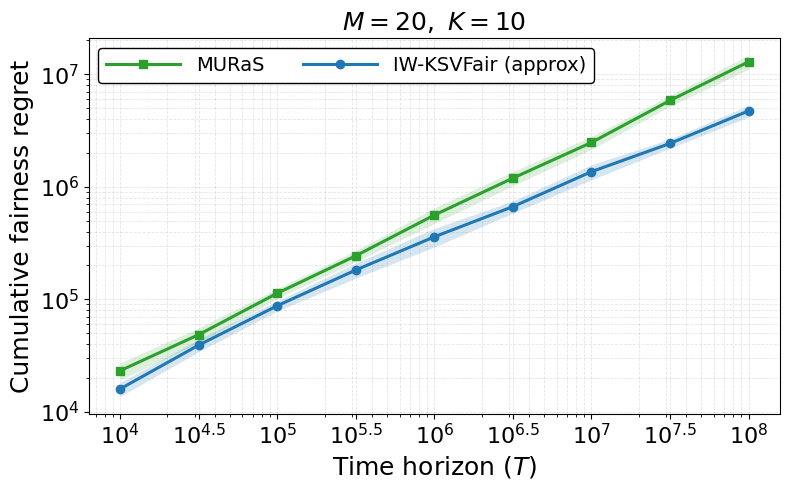}
        \subcaption{$M=20$, $K=10$}
        \label{fig:image3}
    \end{subfigure}
    \hfill
    \begin{subfigure}[t]{0.48\textwidth}
        \centering
        \includegraphics[width=\linewidth]{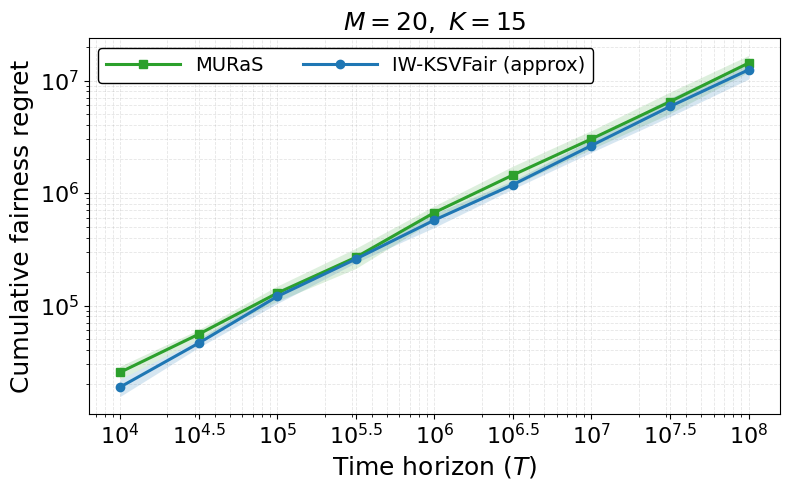}
        \subcaption{$M=20$, $K=15$}
        \label{fig:image4}
    \end{subfigure}

    \caption{Cumulative fairness regret of MURaS and \ouralgo~(approx) on synthetic dataset for $K\in\{3,5,10,15\}$ with $M=20$.}
    \label{fig:different_k}
\end{figure}

\appsubsection{Effect of the Exploration Parameters $\Lambda$ and $\Gamma$}
{appendix: gamma}

Recall that MURaS uses
$
R=\frac{T^{2/3}}{\Lambda}
$
samples for estimating the $K$-Shapley value of each arm, whereas
\ouralgo\ uses the uniform mixing parameter
$
\eta=\frac{M^{2/3}}{\Gamma K^{1/3}T^{1/3}}.
$
Thus, $\Lambda$ and $\Gamma$ control the amount of exploration in
MURaS and \ouralgo, respectively. 
Figure~\ref{fig:gamma} evaluates the sensitivity of the two algorithms to these parameters on the synthetic dataset with $M=20$ and $K=5$.
The parameter values are chosen such that the extents of exploration
in the two algorithms are comparable, thus ensuring a fair
comparison. In particular, MURaS performs $2MR$ full bandit queries
during its exploration phase, while $\eta T$ is the expected number of
queries corresponding to the uniform exploration component of
\ouralgo.

Figure~\ref{fig:gamma} illustrates the
exploration estimation trade-off controlled by $\Lambda$ and
$\Gamma$. For MURaS, a small value of $\Lambda$ results in a long
uniform exploration phase and therefore incurs high exploration
regret. On the other hand, a very large value of $\Lambda$ provides
fewer observations for estimating the arms' $K$-Shapley values, which
can reduce the accuracy of the committed policy. As shown in
Figure~\ref{fig:gamma}(a), $\Lambda=8$ achieves low
regret across the considered time horizons.

For \ouralgo, a small value of $\Gamma$ results in a large value of
$\eta$ and hence excessive uniform mixing. In contrast, a very large
value of $\Gamma$ provides insufficient exploration, which increases
the variability of the importance-weighted estimates. As shown in
Figure~\ref{fig:gamma}(b), $\Gamma=0.8$ provides a balanced trade-off.

Therefore, we use
$\Lambda=8$ and $\Gamma=0.8$ for the synthetic dataset. 
Similarly, we use $\Lambda=100$ and $\Gamma=3$ for SIM dataset. 
It is important to note that $\Lambda$ is not set separately for each
time horizon, neither is $\Gamma$.

\begin{figure}[t]
    \centering

    \begin{subfigure}[t]{0.48\textwidth}
        \centering
        \includegraphics[width=\linewidth]{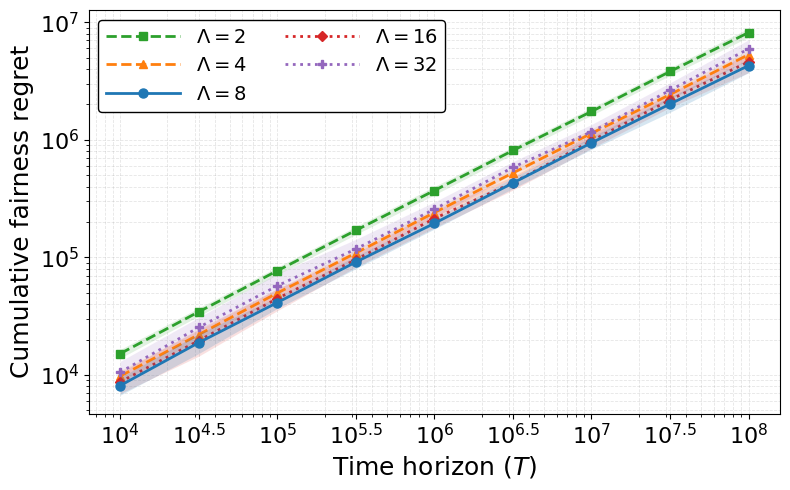}
        \subcaption{MURaS for different values of $\Lambda$}
        \label{fig:subfig1}
    \end{subfigure}
    \hfill
    \begin{subfigure}[t]{0.48\textwidth}
        \centering
        \includegraphics[width=\linewidth]{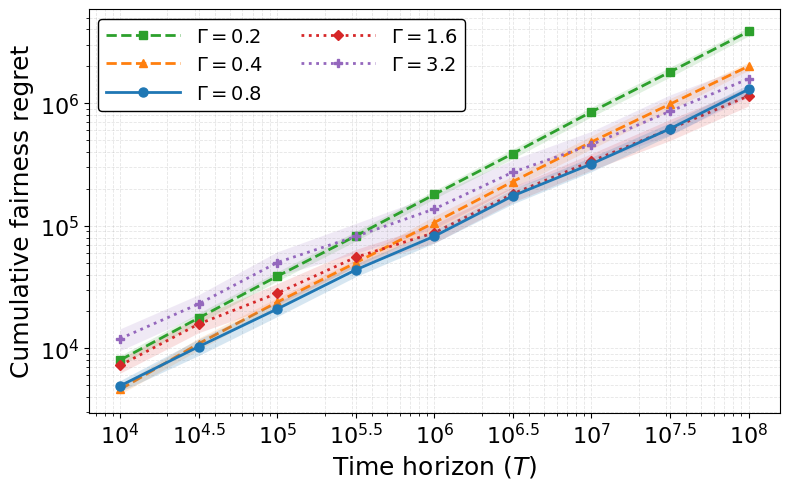}
        \subcaption{\ouralgo~(approx) for different values of $\Gamma$}
        \label{fig:subfig2}
    \end{subfigure}

    \caption{Cumulative fairness regret on the synthetic dataset for different values of $\Lambda$ in MURaS and $\Gamma$ in \ouralgo~(approx), with $M=20$ and $K=5$.}
    \label{fig:gamma}
\end{figure}

\appsubsection{Armwise Merit and Selection Frequency}
{appendix:merit_sel_freq}

Figures~\ref{fig:mainfigure_syn} and~\ref{fig:mainfigure_sim} provides an armwise comparison
between the normalized $K$-Shapley merit and the normalized
empirical selection on synthetic and SIM datasets. The arms are ordered in decreasing order of their true $K$-Shapley values.
Recall that for synthetic dataset, the true $K$-Shapley values are computed exactly, while for SIM, they are computed using Monte Carlo estimation.
For illustration, for synthetic dataset, all arms are shown, and for SIM dataset, 20 randomly chosen arms are shown.
Figures~\ref{fig:subfig1_syn} and~\ref{fig:subfig1_sim} compare the normalized merit with the empirical selection count for the two datasets, respectively. Figures~\ref{fig:subfig2_syn} and~\ref{fig:subfig2_sim} report the corresponding ratios of normalized merit to normalized selection count.
A ratio equal to 1 indicates that the arm is selected exactly in
proportion to its merit. A ratio below 1 indicates over-selection,
since the empirical selection exceeds the merit, while
a ratio above 1 indicates under-selection. 

\ouralgo\ and \ouralgo~(approx) closely track the target merit shares
across the complete merit range, and hence their ratios remain close to
one for almost all arms. MURaS also achieves a close alignment after
completing its uniform exploration phase and committing to the
estimated merit-proportional policy. The deviations observed for NoIW
and NoMix further demonstrate the importance of importance-weighted
bias correction and continued uniform exploration, respectively. 

URS assigns nearly the same selection share to every arm, irrespective
of its merit. Consequently, high-merit arms are under-selected, whereas
low-merit arms are over-selected. Fair-CMAB exhibits a similar mismatch
because its selection frequencies are determined by fixed exposure
requirements rather than the arms' $K$-Shapley values.

ETCG primarily selects the $K$ arms with the highest estimated rewards.
Therefore, its top-$K$ arms are selected substantially more often than
required by their merit shares, resulting in ratios below one. The
remaining arms are selected rarely, resulting in ratios considerably
larger than one. This explains the large values observed for the
lower ranked arms in Figure~\ref{fig:subfig2_syn} and~\ref{fig:subfig2_sim} .

GAP-E allocates most of its samples to arms near the boundary between
the estimated top-$K$ and non-top-$K$ arms, since these arms are the
most difficult to distinguish. As a result, arms near this boundary
are selected more frequently relative to their merit and have ratios
below one. Arms farther from the boundary receive fewer samples and
therefore have ratios above one. This behavior is expected because
GAP-E is designed for top-$K$ identification rather than
merit-proportional selection.

\begin{figure}[h]
    \centering
\vspace{3mm}
    \begin{subfigure}{0.7\textwidth}
        \centering
        \includegraphics[width=\linewidth]{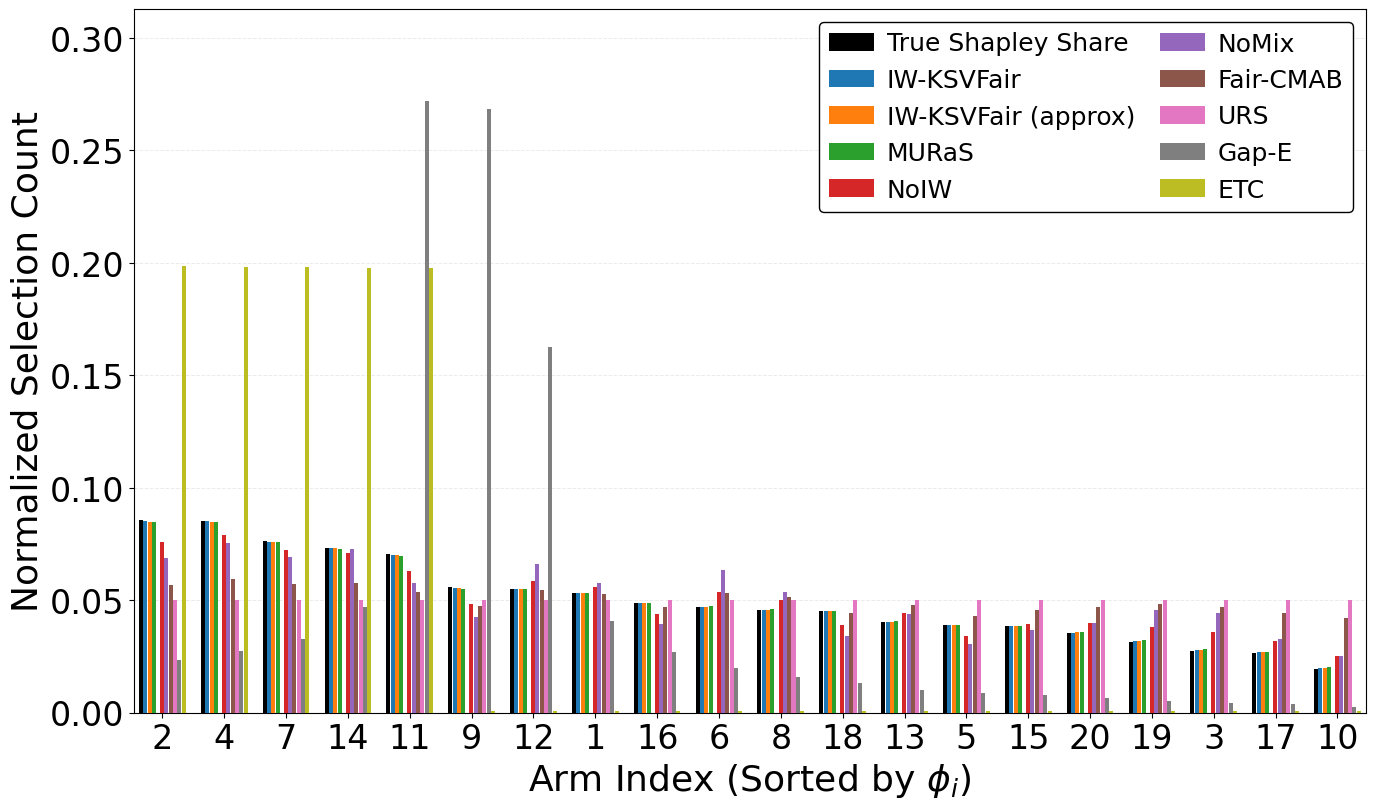}
        \caption{Normalized merit and empirical selection.
        \vspace{5mm}}
        \label{fig:subfig1_syn}
    \end{subfigure}
    \hfill
    \begin{subfigure}{0.7\textwidth}
        \centering
        \includegraphics[width=\linewidth]{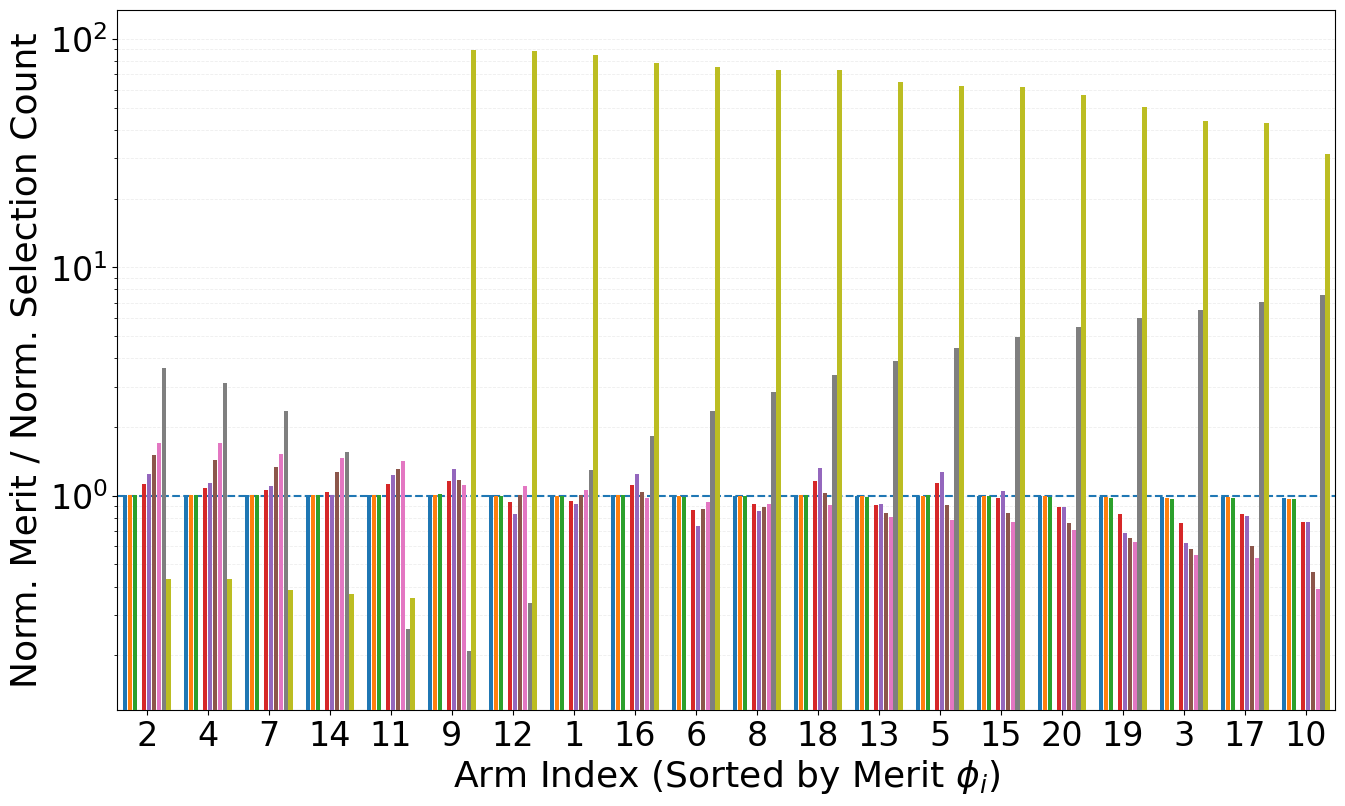}
        \caption{Ratio of normalized merit to normalized selection.}
        \label{fig:subfig2_syn}
    \end{subfigure}

    \caption{Armwise comparison of merit and empirical selection on the
    synthetic dataset.}
    \label{fig:mainfigure_syn}
\end{figure}

\begin{figure}[t]
    \centering

    \begin{subfigure}[t]{0.78\textwidth}
        \centering
        \includegraphics[width=\linewidth]{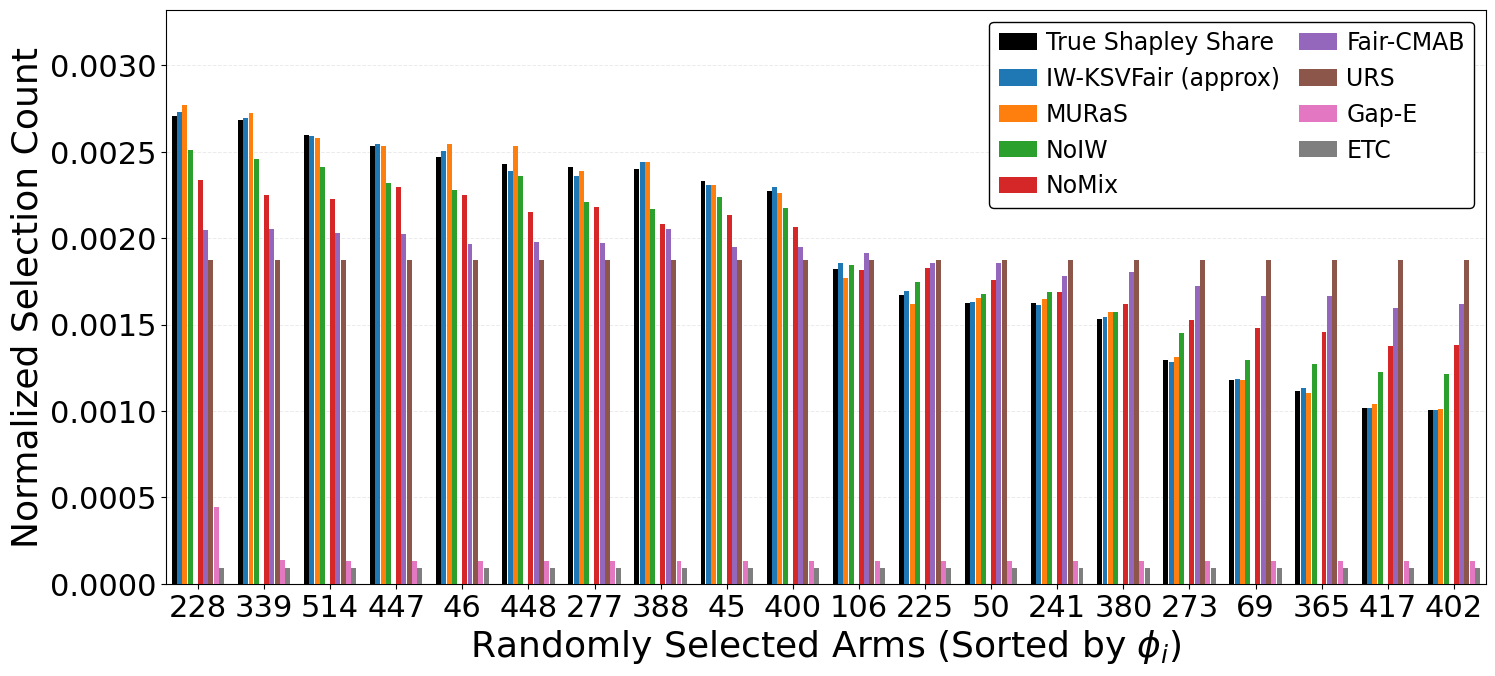}
        \caption{Normalized merit and empirical selection.
        \vspace{5mm}}
        \label{fig:subfig1_sim}
    \end{subfigure}
    \hfill
    \begin{subfigure}[t]{0.78\textwidth}
        \centering
        \includegraphics[width=\linewidth]{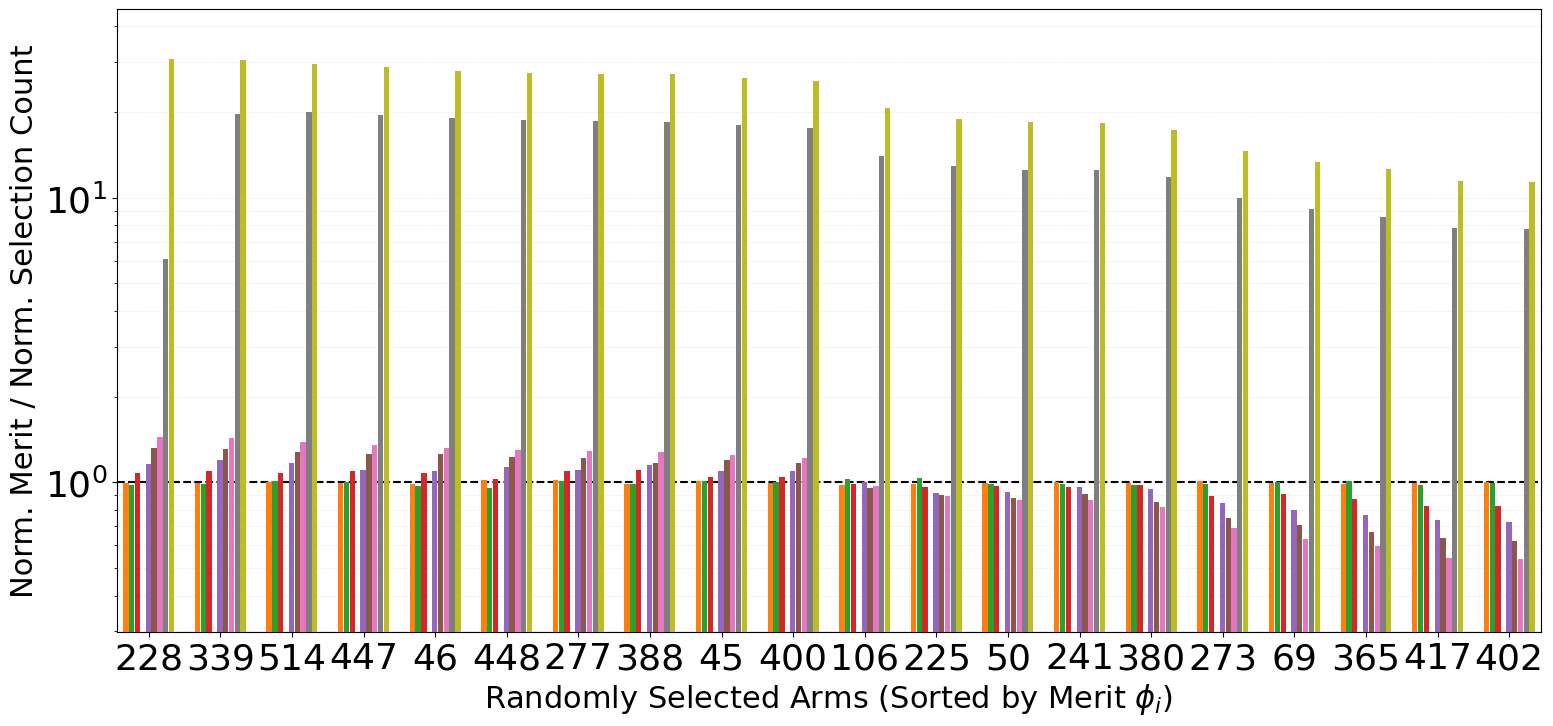}
        \caption{Ratio of normalized merit to normalized selection.}
        \label{fig:subfig2_sim}
    \end{subfigure}

    \caption{Armwise comparison of merit and empirical selection on the
    SIM dataset.}
    \label{fig:mainfigure_sim}
\end{figure}

\end{document}